\documentclass{article}
\usepackage{enumitem}
\usepackage{microtype}
\usepackage{graphicx}
\usepackage{subfigure}
\usepackage{booktabs}
\usepackage{enumerate}
\usepackage{bbm}
\usepackage{comment}
\usepackage{hyperref}
\usepackage{amsmath}
\usepackage{amssymb}
\usepackage{amsthm}
\usepackage{authblk}
\usepackage{natbib}
\usepackage{algorithmic}
\usepackage{algorithm}
\usepackage[dvipsnames]{xcolor}
\usepackage{bbold}
\usepackage{xfrac}
\usepackage{mathtools}
\usepackage{physics}
\usepackage{wrapfig}
\usepackage{nicefrac}
\usepackage{tikz}
\usepackage[margin=1in]{geometry}
\usepackage[capitalize,noabbrev]{cleveref}
\setcitestyle{authoryear,round,citesep={;},aysep={,},yysep={;}}

\hypersetup{
    colorlinks = true,
    linkcolor = {blue},
    citecolor = {blue},
    urlcolor = {blue}
}

\tikzset{
  midarr/.style={decoration={markings,mark=at position #1 with {\arrow{stealth}}},postaction={decorate}},
  midarr/.default=0.5
}

\usetikzlibrary{calc,intersections,decorations.markings}

\theoremstyle{plain}
\newtheorem{theorem}{Theorem}[section]

\newtheorem{lemma}[theorem]{Lemma}
\newtheorem{corollary}[theorem]{Corollary}
\theoremstyle{definition}
\newtheorem{definition}[theorem]{Definition}
\newtheorem{assumption}[theorem]{Assumption}
\theoremstyle{remark}
\newtheorem{remark}[theorem]{Remark}

\newcommand{\twidth}{m_\star}

\newcommand{\bSigma}{\boldsymbol{\Sigma}}

\newcommand{\E}{\mathbb{E}}

\newcommand{\Cov}{\text{Cov}}
\newcommand{\cN}{\mathcal{N}}

\def\sfQ{{\mathsf Q}}
\def\bzero{{\boldsymbol 0}}

\def\betast{\beta_\star}

\DeclareMathOperator*{\argmin}{arg\,min}

\DeclareSymbolFont{rsfs}{U}{rsfs}{m}{n}
\DeclareSymbolFontAlphabet{\mathscrsfs}{rsfs}

\def\bC{{\boldsymbol C}}

\def\bG{{\boldsymbol G}}
\def\bI{{\boldsymbol I}}

\def\bM{{\boldsymbol M}}

\def\bQ{{\boldsymbol Q}}

\def\bT{{\boldsymbol T}}

\def\bW{{\boldsymbol W}}
\def\bX{{\boldsymbol X}}

\def\bZ{{\boldsymbol Z}}

\def\ba{{\boldsymbol a}}
\def\bb{{\boldsymbol b}}

\def \boldm{{\boldsymbol m}}

\def\bv{{\boldsymbol v}}
\def\bw{{\boldsymbol w}}
\def\bx{{\boldsymbol x}}
\def\by{{\boldsymbol y}}
\def\bz{{\boldsymbol z}}

\def\bphi{{\boldsymbol \phi}}

\def\bSigma{{\boldsymbol \Sigma}}
\def\bTheta{{\boldsymbol \Theta}}

\def\spn{{\rm span}}

\def\de{{\rm d}}

\def\de{{\rm d}}

\def\spn{{\rm span}}

\def\cT{{\mathcal T}}

\def\cT{{\mathcal T}}

\def\sP{{\mathsf P}}

\def\cD{{\mathcal D}}

\def\He{{\rm He}}

\def\de{{\rm d}}

\def\bTheta{{\boldsymbol \Theta}}

\def\cT{{\mathcal T}}

\def\bb{{\boldsymbol b}}

\def\bc{{\boldsymbol c}}
\def\bC{{\boldsymbol C}}

\def\sZ{\mathsf Z}

\def\bs{{\boldsymbol s}}

\newcommand{\R}{\mathbb{R}}

\providecommand{\cK}{\mathcal{K}}

\author[1]{Leonardo Defilippis}
\author[2]{Florent Krzakala}
\author[1]{Bruno Loureiro}
\author[1,3]{Antoine Maillard}

\affil[1]{\small D\'epartement d'Informatique, \'Ecole Normale Sup\'erieure, PSL \& CNRS}
\affil[2]{\small Information, Learning and Physics Laboratory, \'Ecole Polytechnique F\'ed\'erale de Lausanne (EPFL)}
\affil[3]{\small INRIA Paris}

\title{Optimal scaling laws in  learning hierarchical multi-index models}

\date{}

\begin{document}

\maketitle

\begin{abstract}
In this work, we provide a sharp theory of scaling laws for two-layer neural networks trained on a class of \emph{hierarchical multi-index} targets, in a genuinely representation-limited regime. We derive exact information-theoretic scaling laws for subspace recovery and prediction error, revealing how the hierarchical features of the target are sequentially learned through a cascade of phase transitions. We further show that these optimal rates are achieved by a simple, target-agnostic spectral estimator, which can be interpreted as the small learning-rate limit of gradient descent on the first-layer weights.  Once an adapted representation is identified, the readout can be learned statistically optimally, using an efficient procedure.
As a consequence, we provide a unified and rigorous explanation of scaling laws, plateau phenomena, and spectral structure in shallow neural networks trained on such hierarchical targets.
\end{abstract}

\section{Introduction}
Despite the staggering practical success of neural networks, we still lack a predictive theory answering a deceptively simple question:
given a structured learning problem, \emph{how does the network adapt to the task, in what order are the relevant features in the data learned, and how these translate in statistical efficiency?} This question sits at the intersection of three active lines of research.

First, the empirical observation of \emph{neural scaling laws} \citep{kaplan2020scaling,brown2020language,hoffmann2022empirical} suggests that the performance of large models scales as a power-law in the training resources, yet---with few exceptions---our mathematical understanding of this relationship remains largely confined to linear models or networks in the lazy regime.
Second, a line of work on the training dynamics of neural networks increasingly suggests that feature learning is not a smooth process, but is associated to long \emph{plateaus and abrupt transitions} in the risk, with features  (or  {\it concepts} in this context) appearing sequentially rather than all at once \citep{saxe2014exact,wei2022emergent,schaeffer2023emergent}. 
Third, empirical analyses of trained networks have uncovered robust regularities in how the learned representations manifest in the trained network weights, such as in their spectral structure, but without a first-principles explanation of \emph{why} particular features are preferred or \emph{when} they  emerge \citep{martin2021implicit,wang2023spectral,thamm2024random}.

This paper provides an end-to-end answer to these questions in a mathematically tractable---yet exhibiting genuine feature learning---setting of a two-layer neural network trained on \emph{hierarchical multi-index} data, a class of structured supervised learning tasks where the target function depends on a hierarchical combination of  functions of the covariates \citep{ren2025emergence,arous2025learning,defilippis2025scaling}.  Here, hierarchical denotes the fact that the directions are  {\it ordered}, in the sense that their relative importance decays as a power-law with their index number. This leads to a
quasi-sparsity of the target representation, with ordered hierarchy of feature strengths as is classical in signal processing \citep{mallat1999wavelet,donoho2006compressed}.

In this task, learning is fundamentally \emph{representation-limited}: the labels $y = g\left(\bW_{\star} \bx\right)$ depend on a low-dimensional but unknown subspace ${\rm span}(\bW_{\star})\subset \mathbb{R}^{d}$ of the input space $\bx\in\mathbb{R}^{d}$, and generalization hinges on discovering this subspace as well as possible from the data.  Our goal is to turn this qualitative picture into sharp, quantitative predictions, showing that sequential feature emergence, spectral outliers, plateaus in the prediction error and scaling laws are consequences of a single mechanism, which we characterize from first principles. Our \textbf{main contributions} are: 
\begin{itemize}[leftmargin=*, itemsep=2pt, topsep=2pt]

\item[(i)] \textbf{Optimal Bayes rates for feature recovery.}  
We derive the exact information-theoretic limits for recovering the features subspace $\spn(\bW_\star)$ from $n=\Theta(d)$ samples. Further, we precisely characterize the associated optimal mean-squared error rates (a.k.a. \emph{scaling laws}) for a large class of hierarchical models, unveiling the presence of cross-overs and plateaus regions that translate a fundamental trade-off between data-scarce and model-limited regions, depending on the hardness of the underlying hierarchical structure. Interestingly, these rates coincide with the minimax bounds for quasi-sparse recovery achieved by the LASSO \citep{raskutti2011minimax}, with previously conjectured results for shallow networks with quadratic activation~\citep{defilippis2025scaling,boncoraglio2025inductive} ---which we thus prove rigorously--- and with recent analysis of SGD in random feature models \citep{ferbach2025dimension}. Our analysis shows that these scalings are in fact universal for a broad class of hierarchical multi-index targets, namely within the generative-exponent-two regime (provided they possess a non-zero second-order Hermite coefficient), well beyond the quadratic activation setting.

\item[(ii)] \textbf{A matching spectral algorithm and sequential feature emergence.}  
We introduce a simple, \emph{target-agnostic} spectral estimator that provably achieves the Bayes-optimal rates derived above. This estimator can be either interpreted as a spectral method applied to the network, or as the small--learning-rate limit of gradient descent on a suitably regularized loss, via the Hessian of a modified objective. As a consequence, gradient-based training is shown to recover the signal subspace optimally, provided it is properly tuned. Moreover, the recovery proceeds sequentially: the $i$-th direction in the hierarchy becomes detectable at a sample complexity $n_i = \Theta(i^{{2\gamma}}d)$ ---where $\gamma$ is the exponent controlling the hierarchy--- leading to a cascade of sharp phase transitions.  This produces plateaus and abrupt drops in the prediction error, providing a theoretical account of the progressive emergence of features observed empirically in neural networks.

\item[(iii)] \textbf{Learnability of the target function by neural networks.}  
Once the features/concepts subspace is identified, we show that learning the second-layer (readout) weights incurs no additional statistical bottleneck: the resulting prediction error matches the Bayes-optimal rate for subspace recovery. Crucially, this two-stage procedure succeeds because the spectral subspace estimator is fully agnostic to the target function and does not require any prior knowledge of its structure or tuning to its decay profile. We prove that this two-stage mechanism—optimal representation discovery followed by efficient readout estimation—applies in the noisy and regularized setting, and to a wide class of  multi-index targets. 

\item[(iv)] \textbf{Impact of noise sensitivity on scaling laws.} We investigate how the noise sensitivity of the target affects the algorithmic recovery problem, altering the sample complexity detection thresholds and the error scaling laws, extending the analysis to the entire class of generative-exponent-two functions, including those with vanishing second-order Hermite coefficient. We find that these new rates are governed by the {\it Noise Sensitivity Exponent} (NSE) of the model, a simple property of the link functions, first introduced by \citet{defilippis2026noise} in the context of single-index models and committee machines. Moreover, since the NSE is not expected to affect information-theoretic recovery, our result points to possible statistical-to-computational gaps in the rates.
\end{itemize}
Together, these results show that representation learning proceeds through a sequence of sharp phase transitions as the number of samples increases. New directions in the signal subspace emerge one after another, leading to plateaus and abrupt drops in the prediction error. This sequential emergence (to borrow the terminology of \citealt{wei2022emergent,schaeffer2023emergent}) of features provides a precise theoretical underpinning for the empirically observed phenomenon of progressive concept learning in neural networks. Strikingly, the resulting phenomenology closely mirrors that recently uncovered for diagonal and quadratic neural networks using heuristic tools from statistical physics, where progressive feature emergence and distinct scaling regimes were observed \citep{defilippis2025scaling}. Our contribution is to place this picture on firm theoretical ground and to substantially generalize it to a broader and more realistic class of models and targets. 

\paragraph{Further related works ---}
Most theoretical work on neural scaling laws focuses on effectively linear models, such as kernel methods and random features, where the generalization behavior can be characterized through spectral properties of fixed representations
\citep{caponnetto2007optimal,spigler2020asymptotic,cui2021generalization,cui2023error,defilippis2024dimension,atanasov2024scaling,bahri2024explaining,maloney2022solvable,paquette2024four,kunstner2025scaling,bordelon2020spectrum}. 
Exceptions have focused either on the joint-training of both layers for linear networks \citep{bordelon2024dynamical,bordelon2025feature,worschech2024analyzing} or non-linear settings with fixed-features \citep{wortsman2025kernel,worschech2024analyzing,braun2025fast}. A central difference is that these works introduce scaling through the covariate distribution, while in our work it is induced by the hierarchical nature of the task. Moreover, our work departs from this literature by addressing scaling laws in a genuinely non-linear, feature-learning regime.  

Closer to us is a recent line of work analyzing two-layer networks with structured first-layer weights and hierarchical multi-index targets \citep{ren2025emergence,arous2025learning,defilippis2025scaling}. In particular, \citet{ren2025emergence,arous2025learning} studied one-pass SGD in this setting, with \citet{arous2025learning} focusing on the case of a quadratic neural network architecture with an idealized, continuous time dynamics. A major difference with this work is that we analyze instead full-batch ERM with a fixed number of samples, rigorously characterizing the optimal sample scaling laws and showing that they can be achieved computationally by a gradient-descent like algorithm, thus providing a fundamental benchmark for sample complexity rates in these problems.

Complementary, \citet{defilippis2025scaling} heuristically derived rates for ERM in quadratic neural networks trained on quadratic targets, building on earlier work by \citet{erba2025nuclear} (see also \citealt{maillard2024bayes,xu2025fundamental,martin2026high}). Our results show that these rates are universal within a large class of target functions (any generative exponent two functions in the language of \citealt{damian2024computationalstatistical}, with non-zero second Hermite coefficient), closing a gap between the purely quadratic setting and general hierarchical multi-index targets. The key  underlying these universal scaling laws is the combination of the \emph{quasi-sparsity} of the target representation, encoded by a heavy-tailed spectrum that induces an ordered hierarchy of feature strengths~\citep{mallat1999wavelet,donoho2006compressed}, and an implicit \emph{rank-sparsity} bias, leading to LASSO-like behavior \citep{raskutti2011minimax}.

Our proof technique is of independent interest and builds on the toolbox of approximate message passing (AMP) and its associated state evolution
\citep{bayati2011dynamics,donoho2013phase,javanmard2013state,berthier2020state,zou2022concise,feng2022,gerbelot2023graph,dudeja2023universality,erba2025nuclear}. Our results leverage the connection between AMP and Bayes-optimal estimation for single- \citep{barbieroptimal2019} and multi-index \citep{aubin2018committee,troiani2024fundamental} models. Similarly, we build up on the literature on spectral methods derived from AMP for single- \citep{mondelli2018fundamental,lu2020phase,maillard2022construction,damian2024computationalstatistical} and multi-index \citep{kovacevic25aspectral,defilippis2025optimalspectral} functions.

\section{Setting}\label{sec:setting}
Consider a supervised regression problem with training data $\mathcal{D}=\{(\bx_{i},y_{i})\in\mathbb{R}^{d+1}:i\in[n]\}$, which we assumed were drawn i.i.d. from a joint distribution over $\mathbb{R}^{d+1}$. Recall that the goal in regression is to learn a target function $f_{\star}(\bx)=\mathbb{E}[y|\bx]$ from the training data $\mathcal{D}$. In the following, we will be interested in particular structured tasks where the dependence of the target $f_{\star}$ on the covariates $\bx\in\mathbb{R}^{d}$ is given by a hierarchical combination of low-dimensional subtasks. Mathematically, this is formalized by \emph{hierarchical multi-index models}.  
\begin{definition}[Hierarchical multi-index model]
\label{assumption:general_setting}
Let $\bW_\star = (\bw^\star_{k}\in\R^d)_{k\in[m_\star]}$ denote a family of $m_{\star}$ orthogonal vectors of norm $\|\bw^\star_k\|^2 = d$, namely $\bW^\star = \sqrt{d}\bQ^\star$, with $\bQ^\star$ uniformly sampled from the Stiefel manifold ${\rm St}(m_\star, d)$. A \emph{hierarchical multi-index} target is defined as
    \begin{equation}
        f_{\star}(\bx) =  \sum_{k=1}^{\twidth} a_k^\star g_k(\langle \bw_{k}^{\star},\bx\rangle),        
    \end{equation}
    where $a^\star_1>a_2^\star>\ldots>a_{\twidth}^\star$ satisfy $\sum_{k=1}^{\twidth}(a_k^\star)^2 = 1$ and $g_{k}:\mathbb{R}\to\mathbb{R}$. Moreover, we say $f_{\star}$ is a \emph{scale-free} hierarchical multi-index model if $a^{\star}_{k}=\Theta(k^{-\gamma})$ for some $\gamma > 0$. 
\end{definition}
A few remarks are in order:
\begin{itemize}
    \item As the name suggests, hierarchical multi-index models can be seen as a sum of $m_{\star}$  effectively one-dimensional tasks $g_{k}(z_{k})$, with decreasing weight $a^{\star}_{k}$. Here, the hierarchy refers only to the relative importance of the features, and not to a compositional structure of the target, as in \citet{cagnetta2024deep,dandi2025computational,tabanelli2026deep}.
    \item More generally, multi-index models are functions of the type $f_{\star}(\bx)=g(\bW_{\star}\bx)$, where $g:\mathbb{R}^{m_{\star}}\to\mathbb{R}$ is known as the \emph{link function} and 
    \begin{equation}\label{eq:def:indices_z}
    \bz=\bW_{\star}\bx\in\mathbb{R}^{m_{\star}}
    \end{equation}
    are known as the \emph{indices}. They have been widely studied as statistical models for regression in the statistics literature \citep{li1991sliced,yuan2011identifiability,babichev2018slice}. 
    \item More recently, multi-index models have gained in popularity in the machine learning theory literature as generative models for data, where they have been used to prove feature learning separation results for neural networks \citep{damian2022neural,abbe23a,dandi2024two}.  
    \item Scale-free hierarchical multi-index models were considered recently in \citet{ren2025emergence,arous2025learning}.  Note that when $\gamma > 1/2$, the target is {\it quasi-sparse} in the index basis \citep{mallat1999wavelet,donoho2006compressed}. This makes the model a suitable framework for investigating feature learning in neural networks, which are known to exhibit an implicit bias towards sparse estimators \citep{gunasekar2017implicit,soudry2018implicit,andriushchenko2023sgd}.
    \item In scale-free hierarchical multi-index models with $\gamma < 1/2$, the coefficients must scale as $a_k^\star \asymp \twidth^{\gamma - 1/2}k^{-\gamma}$ to guarantee the condition $\sum_{k}(a_k^\star)^2 = 1$.
\end{itemize}
We consider a supervised learning task with data generated from a hierarchical multi-index model.
\begin{definition}
\label{def:data_setting}
The training data $\mathcal{D}=\{(\bx_{i},y_{i})\in\mathbb{R}^{d+1}:i\in[n]\}$ is drawn i.i.d. from a scale-free hierarchical multi-index model $f_\star(\bx)$ as in Def.~\ref{assumption:general_setting}. More precisely
\begin{equation}\label{eq:def:teacher_model}
    y_i = f_\star(\bx_i) + \sqrt{\Delta}\xi_i,
\end{equation}
where $\bx_{i}\sim\mathcal{N}(0,\sfrac{1}{d}\bI_{d})$, $\Delta > 0$ is the noise variance and $\xi_i\sim\cN(0,1)$. Furthermore, we consider the high-dimensional limit $n,d\to\infty$ with $\alpha=n/d$ and $m_\star$ fixed. Only after taking this limit do we consider the regime $\alpha,m_\star\gg1$ in which the rates are expressed. In this sense, $m_\star=\Theta_d(1)$ means that $m_\star$ does not grow with $d$.
\end{definition}
Moreover, our results depend on the following additional assumption on the link functions $g_k$.
\begin{assumption}
\label{assumption:gen_exp}
For each index $k\in[m_{\star}]$, $g_k$ is an even function. Moreover, for $z\sim\cN(0,1)$, for some constants $b,C,D>0$, independent of $k, n, d, \alpha, m_\star$,
    \begin{align}\label{eq:boundedness_gk}
    \begin{dcases}
        &|g_k(z)| \leq C(1+|z|^b),\\ 
        &D<\E_z[g_k^2(z)] < C,\quad D<|\E_z[g_k(z)(z^2-1)]| < C.
    \end{dcases}
    \end{align}
\end{assumption}
\begin{remark}
The parity assumption on $g_k$ ensures that learning $\spn(\bW_\star)$ is \emph{non-trivial} \citep{troiani2024fundamental}. It rules out linear correlations that would yield weak recovery at arbitrarily small sample complexity, a regime where the emergence phenomenon cannot occur. Indeed, given $\bz$ as in eq. \eqref{eq:def:indices_z} and $y$ as in eq. \ref{eq:def:teacher_model}, when $\E[\bz\mid y]=0$ a.s.—which holds here by parity—no efficient algorithm can even weakly recover $\spn(\bW_\star)$ below the critical threshold
\begin{equation}
\alpha_c := \left(\sup_{\bM\in\mathbb S^+_{\twidth},\,\|\bM\|_{\rm F}=1}
\|\E_{y}\bG(y)\bM\bG(y)\|_{\rm F}\right)^{-1},
\end{equation}
with $\bG(y):=\E[\bz\bz^\top-\bI_{\twidth}\mid Y=y]$
\citep{barbieroptimal2019,mondelli2018fundamental,lu2020phase,damian2024computationalstatistical,troiani2024fundamental}.

The bound of eq.~\eqref{eq:boundedness_gk} on $\E_z[g_k^2(z)]$ controls the label variance, while the lower bound on
$\E_z[g_k(z)(z^2-1)]$ ensures detectability in the proportional regime $n=\Theta(d)$
(equivalently, a generative exponent equal to $2$), so that the relative difficulty of each index is governed solely by the decay of $a_k^\star$. 
\end{remark}

Hierarchical multi-index models with generative exponent equal to 2 include the vast majority of cases of interest, while the class of models with exponent larger than $2$ mostly includes fine-tuned examples \citep{barbieroptimal2019,damian2024computationalstatistical}. Specifically, Assumption \ref{assumption:gen_exp} considers a subset of the targets with a generative exponent equal to two -- namely, all link functions with a non-zero second-order coefficient in their Hermite expansion. According to the recent framework of \citet{defilippis2026noise}, this corresponds to a {\it noise sensitivity exponent} equal to one. Within this class, we prove the universality of the rates obtained for quadratic models in \citet{defilippis2025scaling}. In Section \ref{sec:new_NSE} we relax this assumption to encompass all hierarchical models with a generative exponent of two, allowing us to explore how generic noise sensitivity exponents affect the error rates. 

Our key goals are twofold: (i) to characterize what are the best achievable rates (a.k.a. scaling laws) in recovering concepts and features for  this class of hierarchical tasks; 
(ii) to show that out-of-the box models such as neural networks  can learn efficiently the features, and the entire functions as well, with those optimal scaling rates.

In order to quantify the recovery of each individual index, or concept, we define the following metric.
\begin{definition}[Matrix-MSE] Let $\bw^{\star}_{k}\in\mathbb{S}^{d-1}(\sqrt{d})$ denote one of the target indices $k\in[\twidth]$. We define the matrix-MSE associated to a predictor $\bw\in\mathbb{R}^{d}$ as: 
   \begin{equation}
   {\rm mse}_k(\bw) := \frac{1}{d^2}\E\left[\|\bw\bw^{\top} - \bw^\star_k\bw_k^{\star\top}\|_{\rm F}^2\right].
   \end{equation}
In our results, we consider this metric in the limit $d\to\infty$, which we assume to exist.
\end{definition}
Additionally, we characterize the {\it weak recovery} transition for the index $k$, {\it i.e.} the minimum number of samples required by an estimator to correlate non-trivially with a specific concept.
\begin{definition}[$k-$critical threshold]\label{def:k_critical} Given an estimator $\hat{\bw}_k \in\R^{d}$ of the signal direction $\bw_k^\star$ (i.e. a measurable function of the training data $\bX,\by$), the $k-$critical threshold is defined as the minimum sample complexity such that the estimator exhibits a finite overlap with $\bw_{k}^\star$, namely  
    \begin{equation}
     \alpha_k :=   \inf\left\{\alpha>0:  \E\left[\frac{| \langle\hat{\bw}_k,\bw_{\star,k}\rangle|}{\|\hat{\bw}_k\| \|\bw_{\star,k}\|}\right] = \Theta_d(1)\right\}. 
    \end{equation}
\end{definition}
The recovery of $\spn(\bW^\star)$ is assessed by the sum of the squared errors for each individual direction, weighted by the contribution of each concept to the target variance.
\begin{definition}[Weighted MSE] \label{def:weighted_MSE} Let $\bW \in \R^{\twidth\times d}$, with rows $\bw_1,\ldots,\bw_{\twidth}$. We define the weighted mean-squared error as
    \begin{equation}
    \label{eq:weightedMSE}
        {\rm MSE}_\gamma(\bW) := \min_{\pi\in\mathfrak S_{\twidth}}\sum_{k=1}^{\twidth} (a_k^\star)^2 {\rm mse}_k(\bw_{\pi(k)}),
    \end{equation}
    where $\mathfrak S_{\twidth}$ is the set of permutations of $[\twidth]$.
\end{definition}
The minimum over permutations accounts for the fact that multi-index models are identifiable only up to a relabeling of the indices: an estimator is not required to return the directions in the order of the hierarchy.
For the second goal --- showing that neural networks can efficiently learn $f_{\star}$ --- we will focus on two-layer neural networks:
\begin{equation}\label{eq:def:model_NN}
    f(\bx;\bTheta) := \frac{1}{\sqrt{p}}\ba^\top\sigma(\bW\bx + \bb),
\end{equation}
with weights $\bTheta:=\left(\ba\in\R^p,\,\bW\in\R^{p\times d}\right)$ and a generic activation function $\sigma:\R\to\R$. As we shall see, two-layer neural networks can agnostically learn with respect to the model, i.e. without knowledge of the individual tasks $g_k$ nor the details in Def.~\ref{assumption:general_setting}. We quantify the generalization capacity of this model through its excess risk
\begin{equation}\label{eq:def:risk}
        R(\bTheta) \coloneqq \E[(f_\star(\bx) - f(\bx;\bTheta))^2].
\end{equation}

\section{Main results}
\label{sec:mainres}
We now discuss our main contributions. First, we derive the information-theoretical optimal reconstruction error for the class of scale-free hierarchical multi-index models introduced in Def. \ref{def:data_setting} under Asm. \ref{assumption:gen_exp}, as quantified by the weighted reconstruction error in Eq.\eqref{eq:weightedMSE}. This result will serve as a fundamental benchmark on the difficulty of subspace recovery for a given $\gamma>0$. Second, we construct a spectral estimator for subspace recovery, and prove that this algorithm achieves the aforementioned optimal reconstruction error. Then, we consider two-layer neural networks, proving that a two-step training procedure also achieves the optimal scaling laws for the excess risk. Finally, relaxing Asm. \ref{assumption:gen_exp}, we derive the optimal algorithmic rates for a broader class of targets, highlighting their connection to the function's noise sensitivity exponent.

\subsection{Bayes-optimal rates for feature recovery}\label{sec:IT}
We  characterize the information-theoretic limits for the recovery of the subspace $\spn(\bW_\star)$ in terms of the decay of the target coefficients $\ba^\star$, the sample complexity $\alpha$ and the subspace dimension $m_\star$.
\begin{definition}[Optimal MSE]
\label{def:MMSE}
Let $\mathcal{D}=(\bX,\by)$ be drawn as in def. \ref{def:data_setting}. Then, the optimal (weighted) mean-square error is achieved by the posterior average:
    \begin{align*}
        {\rm MMSE}_\gamma := \sum_{k=1}^{\twidth} \frac{(a_k^\star)^2}{d^2}\E\left[\|\E[\bw_k\bw_k^\top\,|\,\cD] - \bw_k^\star\bw_k^{\star\top}\|^2_{\rm F}\right]
    \end{align*}
In our results, we consider this metric in the limit $d\to\infty$, which we assume to exist.
\end{definition}
By definition, ${\rm MMSE}_\gamma$ is a lower-bound for the smallest achievable weighted mean-squared error ${\rm MSE}_\gamma(\hat{\bW})$ by any estimator $\hat{\bW}$ that is a function of the dataset $\cD$. Our first main result quantifies precisely the rates of ${\rm MMSE}_\gamma$ as $\alpha\gg 1$, which by construction define the optimal scaling laws for subspace reconstruction in the class of scale-free hierarchical multi-index models. We then have the following theorem (proven in Appendix \ref{app:sec:IT}).
\begin{theorem}[Optimal scaling laws]\label{theorem:bayes_rates}
In the setting of Definitions \ref{assumption:general_setting}, \ref{def:data_setting}, under Assumption \ref{assumption:gen_exp}, for $\alpha,\twidth\gg 1$, the Bayes-optimal mean-squared error satisfies
  \begin{equation}
        {\rm MMSE}_\gamma = \Theta_{\alpha,m_{\star}} \begin{cases}
            \alpha^{-1+\sfrac{1}{2\gamma}},&\gamma > 1/2,\;\alpha\ll   \twidth^{2\gamma},\\
            \twidth/\alpha,&\gamma >1/2,\;   \twidth^{2\gamma}\ll \alpha,\\
           1,&\gamma <1/2,\;\alpha\ll   \twidth,\\
              \twidth/\alpha,&\gamma < 1/2,\;   \twidth\ll \alpha.\\
        \end{cases}
    \end{equation}
    Moreover, the $k$-critical threshold of the Bayes estimator satisfies (with 
    $x_+ \coloneqq \max(0, x)$).
    \begin{equation}
        \alpha^{\rm Bayes}_k = \Theta_k(k^{2\gamma}m_\star^{(1-2\gamma)_+}),
    \end{equation}
\end{theorem}

\subsection{Optimal agnostic subspace recovery}
\label{sec:results:part2}
While \cref{theorem:bayes_rates} provides a fundamental benchmark, an equally important question is whether the optimal reconstruction rates can be \emph{efficiently} achieved by an algorithm which is agnostic to the underlying data distribution. In this section, we give an affirmative answer to this question by constructing an explicit spectral method achieving these rates.
\begin{definition}[Spectral estimator]\label{def:spectral_estimator}
    Given a {\it pre-processing} function $\cT:\R\to\R$ on the labels, consider the symmetric random matrix
\begin{equation}\label{eq:def:spectral_method_matrix}
    \bT = \sum_{i=1}^n \cT(y_i)\bx_i\bx_i^\top,
\end{equation}
with spectrum $\lambda_1\geq\lambda_2\geq\ldots\geq\lambda_d$. Define $r:=\min\{i\in\mathbb N\,\vert\, \lambda_{i+1}-\lambda_{i+2} < C/\sqrt{d}\}$, for some arbitrary constant $C$. We define the spectral estimator $\hat\bW^{\rm sp} = (\hat\bw_1,...,\hat\bw_r)^\top\in\mathbb{R}^{r\times d}$,
where $\hat \bw_{k}$ is the $k\in[r]$ eigenvector of $\bT$ corresponding to $\lambda_k$, normalized such that $\lVert\hat{\bw}_i\rVert^2 = d$. 
\end{definition}
Such estimators have been studied in the context of multi-index models \citep{kovacevic25aspectral,defilippis2025optimalspectral} and proven to achieve the computationally optimal weak recovery threshold \citep{troiani2024fundamental}, for a tailored choice of the pre-processing function $\cT$. Instead, here we consider $\cT$ to be data-agnostic, satisfying only the following mild conditions.
\begin{assumption}\label{assumption:preprocessing}
    Given $Y \sim \cN\left(\sum_{k=1}^{\twidth}a_k^\star g_k(z_k),\Delta\right)$ with $\bz\sim\cN(\bzero,\bI_{m_{\star}})$, we assume that $\cT$ is bounded, strictly increasing and $\inf_y\cT(y)<0<\sup_y\cT(y)=:\tau$.
\end{assumption} 
Under Assumption \ref{assumption:preprocessing}, \citet{kovacevic25aspectral} show that, in the proportional high-dimensional limit $n,d\to \infty$ with $n=\Theta(d)$, the empirical spectrum of $\bT$ in eq.~\eqref{eq:def:spectral_method_matrix} converges to a bounded continuous density --- referred to as the {\it bulk} --- plus $r \leq m_\star$ isolated eigenvalues --- referred to as the {\it spikes} --- with $r$ depending on the sample complexity $\alpha=\sfrac{n}{d}$. This phenomenon is reminiscent of the classical Baik–Ben Arous–Péché (BBP) transition in random matrix theory \citep{baik2005phase}.

The spikes are separated from the bulk by a finite gap and their corresponding eigenvectors exhibits a non-trivial correlation with $\spn(\bW_\star)$. Therefore, the criterion described in Definition \ref{def:spectral_estimator} (reminiscent of the heuristic ``elbow'' method in principal components analysis, see Fig. \ref{fig:elbow}) for selecting the number of eigenvalues agnostically retains only the informative components, discarding the noise contributions associated with the bulk. Note also that the choice of cutoff $C/\sqrt{d}$ is arbitrary, as in general we want to separate the bulk with spacing $o_d(1)$ and the spikes with spacing $\Theta_d(1)$.

\begin{remark}[GD interpretation of spectral method]
    Note that this spectral method also admits an interpretation in terms of a gradient descent algorithm. Indeed, considering a slightly modified loss than the pure square one, this matrix is nothing but the Hessian of the loss \citep{bonnaire2025role}.
The early dynamics of GD on the first layer of (\ref{eq:def:model_NN}), for weights initialized on a sphere with small radius, will thus follow the proposed spectral method. This construction was recently used in \citet{zhang2025neuralnetworks} to show how two-layer neural networks identify features within the first few GD steps in the first layer weights. More precisely, they show that the early dynamics of a suitable gradient-based algorithm, reads as a power iteration
\begin{align}
   \bw_k^{t+1} &= \bw^t_k -\nabla_{\bw_k} \frac{1}{n}\sum_{i=1}^n \ell\left(y_i, \ba^\top \sigma(\bW^t\bx_i)/\sqrt{p}\right)
   \approx (\bI_d - a_k\sigma''(0) \bT/\sqrt{p}) \bw_k^t,
\end{align}
where $\bT$ has the structure defined in eq. \eqref{eq:def:spectral_method_matrix} with pre-processing $\cT = \ell'(\,\cdot\,, 0)$, and $\ell'$ denotes the derivative of $\ell$ with respect to its second argument.
\end{remark}
\begin{figure}
    \centering
    \includegraphics[width=.6\linewidth]{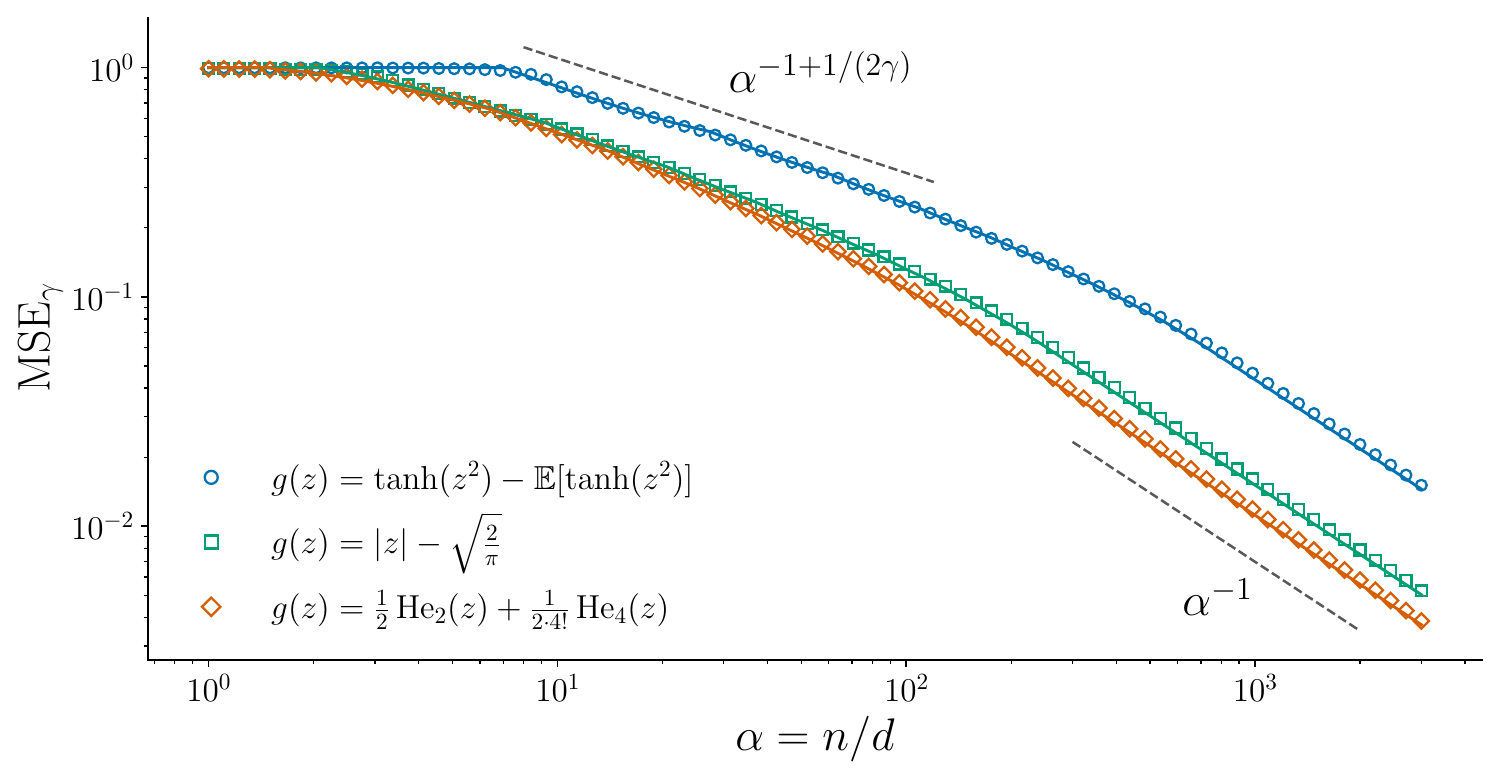}
    \caption{Weighted mean square error ${\rm MSE}_\gamma$ -- see Def.~\ref{def:weighted_MSE} -- of the spectral estimator  of Def.~\ref{def:spectral_estimator} with preprocessing function $\cT(y)=y/(1+|y|)$. The target is given by the hierarchical multi-index model ~\ref{assumption:general_setting}, with $g_k(z)=g(z)\; \forall k$, stated in the legend, and $a_k^\star \propto k^{-\gamma}$, $\gamma = 1$. Solid lines: asymptotic prediction for $d\to\infty$, from Theorem~\ref{app:thm:kovacevic_spectral} and the overlaps derived in App.~\ref{app:sec:spectral}. Dots represent averages over $30$ instances, for covariates with dimension $d = 2000$, feature space dimension $m_\star = 10$ and label noise variance $\Delta = 0.5$.}
    \label{fig:rates}
\end{figure}

Our second main theoretical result shows that the data-agnostic spectral method defined above indeed achieves the optimal reconstruction rates of \cref{theorem:bayes_rates}.
\begin{theorem}\label{thm:spectral_rates}
    In the setting of Definitions \ref{assumption:general_setting}, \ref{def:data_setting}, under Assumption \ref{assumption:gen_exp}, denote $\hat{\bW} = (\hat{\bW}^{\rm sp},\bzero_{(\twidth-r)\times d})\in\R^{\twidth\times d} $, where $\hat{\bW}^{\rm sp}\in\R^{r\times d}$, $r\leq \twidth$, is the spectral estimator  defined in \ref{def:spectral_estimator}, with $\cT$ satisfying Assumption \ref{assumption:preprocessing}.
    Then, for $g_k$ strictly increasing on $(0,\infty)$, and $\alpha,m_\star\gg 1$,
    \begin{equation}
       {\rm MSE}_\gamma(\hat\bW) = \Theta_{\alpha, m_{\star}}({\rm MMSE}_\gamma),
    \end{equation}
    and the spectral estimator's $k$-critical threshold $\alpha^{\rm sp}_k$ satisfies $\alpha^{\rm sp}_k= \Theta_k(\alpha^{\rm Bayes}_k)$.
\end{theorem}

\begin{remark}
    Note that, as the spectral method \ref{def:spectral_estimator} retrieves $r\leq \twidth$ directions, we construct the full estimator by filling the remaining columns with zeros to evaluate ${\rm MSE}_\gamma$ in Def. \ref{def:weighted_MSE}. This zero-padding is harmless: the added zeros do not impact the derived scaling laws or the estimator's agnostic nature. Moreover, the spectral thresholds depend on the link functions $g_k$, so that the spikes do not necessarily emerge in the order of the hierarchy, and the $k^{\rm th}$ eigenvector of $\bT$ can be aligned with $\bw^\star_h$, $h\neq k$. This is accounted for by the minimum over permutations in Definition \ref{def:weighted_MSE}, and $\alpha_k^{\rm sp}$ denotes the threshold at which an eigenvector aligned with $\bw_k^\star$ emerges.
\end{remark}
\begin{remark}
    The additional technical assumption of $g_k$ strictly increasing on $(0,\infty)$ ensures that the informative outliers emerge from the right edge of the bulk. If we retain the unrestricted class of admissible link functions, the theorem continues to hold after a minor extension of the proof that accounts for outliers emerging from both edges of the bulk. This extension does not change any of the stated thresholds, rates, or scaling exponents.
\end{remark}

In Appendix \ref{app:sec:spectral}, we show that the rates in Theorem \ref{thm:spectral_rates} allow for an interpretation reminiscent of the "universal error decomposition" in \citet{defilippis2025scaling}. Let $k_\alpha \asymp \min(\alpha^{1/(2\gamma)},\twidth)$  denote the number of weakly learnable features at sample complexity $\alpha$ for $\gamma>\sfrac{1}{2}$ -- corresponding to the number of spikes of $\bT$ eq. \eqref{eq:def:spectral_method_matrix}. Then, the spectral mean-squared error decomposes into two main contributions: (i) the approximation error of the learnable concepts, scaling as $k_\alpha \cdot \Theta(\alpha^{-1})$; (ii) the underfitting due to the unlearnable concepts, scaling as $\Theta\left(\sum_{k=k_\alpha+1}^{m_\star} (a_k^\star)^2 \right)$.

Figure \ref{fig:rates} shows the decay of the weighted MSE for three scale-free hierarchical model variants as a function of the sample complexity $\alpha$, based on finite-size experiments, using the spectral method introduced in Definition \ref{def:spectral_estimator}. Notably, around $\alpha \asymp m_\star^{2\gamma} \approx 100$, we observe a crossover between the rates $\alpha^{-1+1/(2\gamma)}$ and $m_\star/\alpha$, as predicted by Theorem \ref{thm:spectral_rates}. Figure \ref{fig:spectra} illustrates the sequential emergence of new concepts as the sample complexity grows, visualized as spikes detaching from the eigenvalue bulk of the matrix $\bT$ in eq.~\eqref{eq:def:spectral_method_matrix}.

\begin{algorithm}[t]
   \caption{Spectral Initialization and Ridge Training}
   \label{alg:network_training}
\begin{algorithmic}
   \STATE {\bfseries Input:} Dataset $\cD = (\bX, \by)\in\R^{n\times d}\times\R^{n}$, hidden layer width $p$, regularization $\lambda$.
   \STATE {\bfseries Output:} Network parameters $\bW \in \R^{p\times d}$, $\hat{\ba} \in \R^p$.
   
   \STATE
   \STATE \textit{1. Data Splitting}
   \STATE Partition the dataset $\cD$ into two disjoint sets $\cD_1$ and $\cD_2$ such that $|\cD_1| = |\cD_2| = \frac n2$.
   
   \STATE
   \STATE \textit{2. Feature Learning (Spectral Initialization on $\cD_1$)}
   \STATE Compute the spectral estimator $\hat{\bW}^{\rm sp}\in\R^{r\times d}$ of Def. \ref{def:spectral_estimator} using $\cD_1$.
   \STATE Sample a random matrix $\bZ \in \R^{p \times r}$ with i.i.d. entries drawn from $\cN(0,r^{-1})$.
   \STATE Set first-layer weights: $\bW = \bZ\hat{\bW}^{\rm sp}$.
   
   \STATE
   \STATE \textit{3. Readout Training (Ridge Regression on $\cD_2$)}
   \STATE Compute the feature matrix $\Psi \in \R^{n/2 \times p}$ on $\cD_2$, where $\Psi_{ij} = \sigma(\langle\bw_j,\bx_i^{(2)}\rangle+b_j)/\sqrt{p}$,  $b_j\sim\cN(0,1)$.
   \STATE Solve for the second-layer weights:
   \STATE \quad $\hat\ba = (\Psi^\top\Psi + \frac{n}{2}\lambda\bI_p)^{-1}\Psi^\top \by^{(2)}$
   
   \STATE
   \STATE \textbf{return} $\bW, \hat{\ba}$
\end{algorithmic}
\end{algorithm}

\begin{figure*}[t]
    \centering
    \includegraphics[width=1\linewidth]{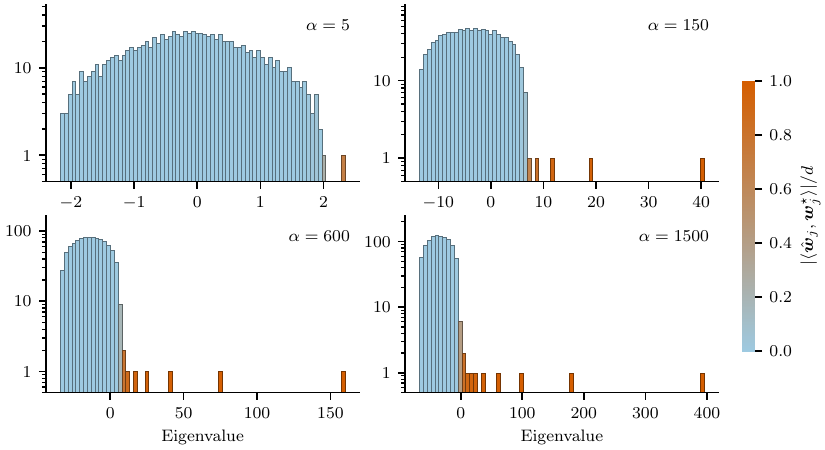}
    \caption{Empirical spectrum distribution of the matrix $\bT$ defined in eq. \eqref{eq:def:spectral_method_matrix}, with preprocessing function $\cT(y)=y/(1+|y|)$, for a single realization at increasing sample complexity $\alpha$ (indicated in each panel), highlighting the sequential emergence of concepts (spikes) from the bulk. Each bin is colored by the overlap $|\langle \hat\bw_j, \bw^\star_j\rangle|/d$ between the eigenvector $\hat\bw_j$ of the $j^{\rm th}$ largest eigenvalue, normalized as in Def.~\ref{def:spectral_estimator}, and the target direction $\bw^\star_j$ (averaged over the eigenvalues in the bin, and set to zero for $j>\twidth$). The target is given by a hierarchical multi-index model~\ref{assumption:general_setting}, with $g_k(z) = \frac{1}{2}\He_2(z)+\frac{1}{2\cdot4!}\He_4(z)\;\forall k$, $a_k^\star \propto k^{-\gamma}$ with $\gamma = 1.3$ and $\sum_k (a_k^\star)^2=1$, and $\Delta = 0.5$. The covariates dimension is $d = 1000$, the feature space dimension is $\twidth = 20$.} 
    \label{fig:spectra}
\end{figure*}

\subsection{Learning hierarchical multi-index models with neural networks}\label{sec:readout}
Finally, we turn our attention to the problem of learning hierarchical multi-index target functions with a two-layer neural network. In particular, we are interested in characterizing the scaling laws associated with the excess risk in \cref{eq:def:risk} (rather than the reconstruction error in \cref{eq:weightedMSE}). Surprisingly, our main result in this section will show that the excess risk associated with a suitable training procedure described in Alg. \ref{alg:network_training}, achieves the optimal rates in \cref{theorem:bayes_rates} (proof in App.\ref{app:sec:readout}): 

\begin{theorem}\label{thm:readout}
   Under the assumptions of Theorem \ref{thm:spectral_rates}, let each $g_k$, $k\in[m_\star]$, be weakly differentiable with $\E_z[g_k'(z)^2]\leq C'$, for a constant $C'$ independent of $k$, and let $\sigma$ be bounded, continuous and nonconstant. For every $\alpha$, there exist a width $p(\alpha)$ and a regularization $\lambda(\alpha)>0$ such that, for $p\geq p(\alpha)$ and $0<\lambda\leq\lambda(\alpha)$, the excess risk $R_{\rm NN}$ eq.~\eqref{eq:def:risk}, where the expectation is over all sources of randomness, of the two-layer neural network eq.~\eqref{eq:def:model_NN} trained according to Algorithm~\ref{alg:network_training} satisfies
    \begin{align}
        R_{\rm NN} = \Theta({\rm MMSE}_\gamma). 
    \end{align}
\end{theorem}

\begin{remark}
    This final result demonstrates that neural networks can efficiently learn the target function, achieving the exact same rates as those of the optimal feature subspace recovery. The proof effectively reduces the problem to random features regression on the $r$-dimensional projected inputs, whose dimension stays finite as $n,d\to\infty$. 
    We note that our results do not imply that these rates correspond to the decay of the Bayes risk, defined as $\E\left[\left(f_\star(\bx)-\E[f_\star(\bx)|\bx,\cD]\right)^2\right]$, and which constitutes a lower bound for the risk $R(\mathbf{\Theta})$ in Eq. \eqref{eq:def:risk}. 
    Nevertheless, we notice that under the hypothesis of a {\it fully specialized} Bayes-optimal estimator—i.e., assuming $d^{-1}\mathbf{\hat{W}}^{\rm Bayes}(\mathbf{W}^{\star})^\top$ converges to a diagonal matrix in the high-dimensional limit—the risk for a network with matching architecture and Bayes weights, $f(\mathbf{x}) = \sum_{k=1}^{\twidth} a_k^\star g_k(\hat{\bw}^{\rm Bayes}_k)$, achieves the rates established in Theorem \ref{theorem:bayes_rates}. We leave such a precise analysis of the optimal risk as a future direction.
\end{remark}

\subsection{Effects of a generic noise sensitivity exponent}\label{sec:new_NSE}
The results presented in Sections \ref{sec:IT}, \ref{sec:results:part2}, \ref{sec:readout} rely on the condition $\E[g_k(z)(z^2-1)]\neq 0$ -- or, equivalently, a Noise Sensitivity Exponent (NSE) equal to one \citep{defilippis2026noise}, as implied by Assumption \ref{assumption:gen_exp}. This regime ensures easier features detectability and leads to matching bounds between information-theoretic and algorithmic rates. For generic link functions $g:\R\to\R$, the NSE is defined as follows
\begin{definition}[Noise Sensitivity Exponent {\citep{defilippis2026noise}}]\label{def:nse}
    Let $g : \R \to \R$ such that $|g(z)| \leq C (1 + |z|^k)$ for some $C, k > 0$.
    The noise-sensitivity exponent is defined as
    \begin{equation}
    \label{eq:def:nse}
        \beta_\star(g) \coloneq \min\{\beta\in\mathbb N : \E_{z\sim\cN(0,1)}\left[g^\beta(z)(z^2 - 1)\right]\neq 0\}.
    \end{equation}
\end{definition}
While typical functions have $\rm NSE=1$, $\rm NSE=2$ captures the majority of functions with zero second Hermite coefficient, including, for instance, higher-order Hermite polynomials $g(z)=He_{2k}(z)$ with $k>1$. We refer to \citet{defilippis2026noise} for more examples.

Introduced in the context of single-index models and committee machines, the NSE can be shown to govern the optimal computational weak-recovery threshold of features in large noise regimes. In this section, we provide an example showing how a generic NSE can affect the optimal algorithmic weak-recovery.
\begin{assumption}\label{assumption:genericNSE}
    For all $k\in[m_{\star}]$, $g_k = g$ is an even function. Moreover, for $z\sim\cN(0,1)$, for some constants $b,C,D>0$, independent of $k, n, d, \alpha, m_\star$, it satisfies $|g(z)| \leq C(1+|z|^b)$, $D<\E_z[g^2(z)] < C$,
    and the noise sensitivity exponent of $g$ is given by $\beta_\star\in\mathbb N$.
\end{assumption}
In agreement with \citet{aubin2018committee,troiani2024fundamental,defilippis2026noise}, we assess the performance of the Bayes-AMP estimator $\hat{\bW}_{\rm AMP}$, as the optimal estimator among first-order iterative methods \citep{celentano20a}, {\it i.e.} the minimizer of the mean-squared error eq. \ref{eq:weightedMSE}.

\begin{theorem}\label{thm:genericNSE}
    In the setting of Definitions \ref{assumption:general_setting}, \ref{def:data_setting}, under Assumption \ref{assumption:genericNSE} and $1\ll\alpha\ll m_\star^{2\gamma\beta_\star}$, $\gamma > \sfrac{1}{2}$, the Bayes-AMP estimator $\hat{\bf W}^{\rm AMP}$ \citep{aubin2018committee,troiani2024fundamental} achieves
    \begin{equation}
        {\rm MSE}_\gamma\left(\hat{\bf W}^{\rm AMP}\right) = \Theta_{\alpha,m_\star}\left(\alpha^{-\left(1-\sfrac{1}{2\gamma}\right)/\beta_\star}\right).
    \end{equation}
    Moreover, its k-critical threshold $\alpha_k^{\rm AMP}$ satisfies $\alpha_k^{\rm AMP} = \Theta_k\left(k^{2\gamma\beta_\star}\right)$.
\end{theorem}
The proof of this result is discussed in Appendix \ref{app:sec:genericNSE}, in which we also show that the spectral estimator of Definition \ref{def:spectral_estimator} achieves the algorithmic bounds in Theorem \ref{thm:genericNSE}. 

There is an interesting consequence to this result if one follows the conclusion of the statistical mechanics approach to statistical-to-computational gaps~\citep{bandeira2018notes,aubin2018committee,tsirkas2026monotonicity}: while the NSE changes the computational rate for feature emergence, one expects the information-theoretic rates to remain insensitive to the NSE. Indeed the replica-symmetric picture \citep{aubin2018committee,troiani2024fundamental} suggests that information-theoretic recovery is governed by the global maximizer of a functional called the ``replica-symmetric potential'', whereas efficient recovery is controlled by the local instability of a trivial critical point of this potential ~\citep{zdeborova2016statistical,franz1998effective}. As such, a larger NSE not only modifies the computational rates, but also seems to open a genuine statistical-to-computational gap.

\section{Conclusion and limitations}

We established a sharp theory of scaling laws for hierarchical multi-index models in a genuinely representation-limited regime. The main message is that optimal representation learning is sequential: as the sample size grows, the relevant directions become detectable one by one, each through a new outlier in the spectrum, producing plateaus and abrupt drops in prediction error. We derive the corresponding Bayes-optimal rates for subspace recovery and prediction, show that they are achieved by a simple target-agnostic spectral estimator, and prove that once this representation is recovered, the readout can be learned without an additional statistical bottleneck. These results unify neural scaling laws, progressive feature emergence, and the structured spectra observed in trained weights within a single tractable framework, while showing that the power laws previously identified in quadratic settings extend to a (much) broader class of hierarchical multi-index targets.  We also provide exceptions to this universality class. 

Our analysis deliberately focuses on an idealized setting: isotropic Gaussian inputs, two-layer architectures, and high-dimensional asymptotics. These assumptions are a limitation in scope, but also the source of the model's strength. By stripping the problem down to its essential ingredients -- hierarchical target geometry, representation-limited learning, and spectral recovery of the signal subspace -- the framework yields exact, non-perturbative predictions and exposes the mechanisms behind scaling laws and feature emergence. We therefore view this model not as an endpoint, but as a clean foundation on which more realistic data distributions, architectures, and training dynamics, including SGD, can be incorporated systematically.

\section*{Acknowledgments}
We would like to thank Yatin Dandi, Vittorio Erba and Lenka Zdeborov\'a for insightful discussions. BL and LD were supported by the French government, managed by the National Research Agency (ANR), under the France 2030 program with the project references ``ANR-23-IACL-0008'' (PR[AI]RIE-PSAI) and ``ANR-25-CE23-5660'' (MAPLE), as well as the Choose France - CNRS AI Rising Talents program. FK acknowledges funding from the Swiss National Science Foundation grants OperaGOST (grant number 200021 200390) and DSGIANGO (grant number 225837). This work was supported by the Simons Collaboration on the Physics of Learning and Neural Computation via the Simons Foundation grant ($\#1257412$).

\bibliography{main}
\bibliographystyle{abbrvnat}
\newpage
\appendix

\onecolumn
\section{Proof of Theorems \ref{theorem:bayes_rates} and \ref{thm:spectral_rates}}\label{app:sec:IT}

In this section, we prove the information-theoretic results stated in Theorem \ref{theorem:bayes_rates}. In particular, we derive matching bounds for the ${\rm MMSE}_\gamma$, defined \ref{def:MMSE}, and the $k$-critical threshold $\alpha_k$ (Definition \ref{def:k_critical}). In Appendix \ref{app:sec:lower-bound}, we analyze an oracle estimator that exploits additional information to achieve a weighted mean-squared error smaller than ${\rm MMSE}_\gamma$. Finally, in Appendix \ref{app:sec:spectral}, we complete the proof by characterizing the scaling laws and $k$-critical thresholds of the spectral estimator defined in \ref{def:spectral_estimator}, which simultaneously establishes the results in Theorem \ref{thm:spectral_rates}.

\subsection{Lower bound}\label{app:sec:lower-bound}
\begin{definition}[Oracle Estimator] 
    Consider a dataset $\cD = \{(\bx_i,y_i)\in\R^{ d}\times\R\}_{i\in[n]}$, where the labels are generated by a multi-index model $y\sim\sP(\cdot|\langle\bw^\star_1,\bx\rangle, \ldots,\langle\bw^\star_{\twidth},\bx\rangle)$, with weights $\bW^{\star}$ drawn from a distribution $\sP_{\bW}$. We define the Oracle Estimator as the matrix $\bW^{\rm oracle}\in\R^{m_\star\times d}$ with rows
    \begin{equation}
        \bw_k^{\rm oracle} = \argmin_{\bw}\E\left[\|\bw_k\bw_k^\top-\bw_k^\star\bw_k^{\star\top}\|^2_{\rm F}\,|\,\cD,\{\bw_h\}_{h\neq k}\right], k\in[\twidth].
    \end{equation}
\end{definition}
Analogously to the Bayes-optimal estimator case, the weighted mean-squared error ${\rm MSE}_\gamma(\bW^{\rm oracle})$ is lower bounded by
\begin{equation}
    \begin{split}
    {\rm MMSE}^{\rm oracle}_\gamma &= \frac{1}{d^2}\sum_{k=1}^{\twidth}(a_k^\star)^2\,\E\Big[\big\|\E[\bw_k\bw_k^\top|\cD,\{\bw^\star_h\}_{h\neq k}]\\
    &\qquad\qquad-\bw_k^\star\bw_k^{\star\top}\big\|^2_{\rm F}\Big].
    \end{split}
\end{equation}
As a consequence of the law of total variance
\begin{align}
    d^2\,{\rm mmse}_k&:= \E\left[\Cov\left(\bw_k\bw_k^\top\,|\,\cD\right)\right] \\
    &= \E\left[\Cov\left(\bw_k\bw_k^\top\,|\,\cD,\{\bw_h\}_{h\neq k}\right)\right] +\underbrace{ \E\left[\Cov_{\{\bw_h\}_{h\neq k}}\left(\E\left[\bw_k\bw_k^\top \,|\, \cD,\{\bw_h\}_{h\neq k}\right]\right)\right]}_{\geq0}\\
    &\geq d^2\, \E\left[\|\E[\bw_k\bw_k^\top|\cD,\{\bw^\star_h\}_{h\neq k}]-\bw_k^\star\bw_k^{\star\top}\|^2_{\rm F}\right].
\end{align}
Therefore,
\begin{equation}
    {\rm MMSE}_\gamma^{\rm oracle} \leq {\rm MMSE}_\gamma.
\end{equation}
We now focus on the specific setting of our interest defined in Section \ref{sec:setting}. For each index $k\in[\twidth]$, conditioning on $\{\bw_h\}_{h\neq k}$, the problem of the optimal estimation of $\bw_k^\star$ becomes statistically equivalent to the Bayes-optimal estimation given a dataset  $\cD_k :=\{(\bx_\nu,\overline{y}_\nu)\}_{\nu\in[n]}$, where the labels are generated by the single-index model
\begin{equation}
    \overline{y}_i = a_k^\star g_k(\langle\bw_k^\star,\bx_i\rangle) + \sqrt{\Delta}\xi_i,
\end{equation}
where $\xi_i$ is the same label noise as in the original dataset. We can now characterize the oracle estimator using the results from the literature on {\it single-index models}, in particular \citet{barbieroptimal2019}. In Appendix \ref{app:sec:results_sindex} we show that the information-theoretic weak recovery threshold for single-index models is a monotonic decreasing function of the signal-to-noise ratio. Combined with Assumption \ref{assumption:gen_exp}, this implies that the sequence of $k$-critical thresholds $\alpha_k$ is bounded by strictly increasing functions.\footnote{If $g_k = g,\,\forall k$, the sequence $\alpha_k^{\rm oracle}$ itself is strictly increasing.} 
In particular, the first result in Corollary \ref{app:corollary:sindex_thresholds}, for a signal-to-noise ratio $\lambda = (a_k^\star)^2/\Delta$, implies that the sequence is bounded by
\begin{equation}
    \alpha_k^{\rm oracle} = \Theta\left((a_k^\star)^{-2}\right)  =\Theta(k^{2\gamma} m_\star^{(1-2\gamma)_+}),
\end{equation}
with $(x)_+=\max(0,x)$.
Further, denoting by $k_\alpha$ the largest index $k$ such that $\alpha > \alpha_k$, i.e. $k_\alpha = \min(\twidth,\Theta(\alpha^{1/(2\gamma)}))$ for $\gamma > 1/2$ or $k_\alpha = \min(\twidth,\Theta(\twidth^{1-1/(2\gamma)}\alpha^{1/(2\gamma)}))$ for $\gamma <1/2$, and exploiting the second result in Corollary \ref{app:corollary:sindex_thresholds},
\begin{align}
 {\rm MMSE}_\gamma^{\rm oracle} &=  \left(\sum_{k=1}^{k_\alpha}+\sum_{ k=k_\alpha+1}^{\twidth}\right)(a_k^\star)^2\E\left[\|\E[\bw_k\bw_k^\top|\cD,\{\bw^\star_h\}_{h\neq k}]-\bw_k^\star\bw_k^{\star\top}\|^2_{\rm F}\right]\\
 &\geq C\left( \sum_{k=1}^{ k_\alpha} \,\frac{(a_k^\star)^2}{(a_k^\star)^2\,\alpha} + \sum_{k= k_\alpha+1}^{\twidth}(a_k^\star)^2\right)  \label{app:eq:mmse_decomposition}\\
 &=\begin{cases}
     \Theta\left(\alpha^{-1+1/(2\gamma)}\right),&\gamma>1/2,\; \alpha \ll \twidth^{2\gamma}\\
     \Theta\left(\twidth/\alpha\right),&\gamma>1/2,\; \alpha \gg \twidth^{2\gamma}\\
     \Theta\left(1\right),&\gamma<1/2,\; \alpha \ll \twidth \\
     \Theta\left(\twidth/\alpha\right),&\gamma<1/2,\; \alpha \gg \twidth
     \end{cases}
\end{align}
Note that we have used 
\begin{equation}
    \sum_{k=k_\alpha + 1}^{\twidth}k^{-2\gamma }= \begin{cases}
        \Theta(k_\alpha^{1-2\gamma}),&\gamma > 1/2,\\
        \Theta(\twidth^{1-2\gamma}),&\gamma < 1/2.
    \end{cases}
\end{equation}
In eq. \eqref{app:eq:mmse_decomposition}, the first term is a lower bound to the (weighted) mean-squared error of weakly recovered features, while the second corresponds to the underfitting contribution of the unlearned ones.

\subsection{Upper bound -- (Theorem \ref{thm:spectral_rates})}\label{app:sec:spectral}

\begin{figure}
    \centering
    \includegraphics[width=0.7\linewidth]{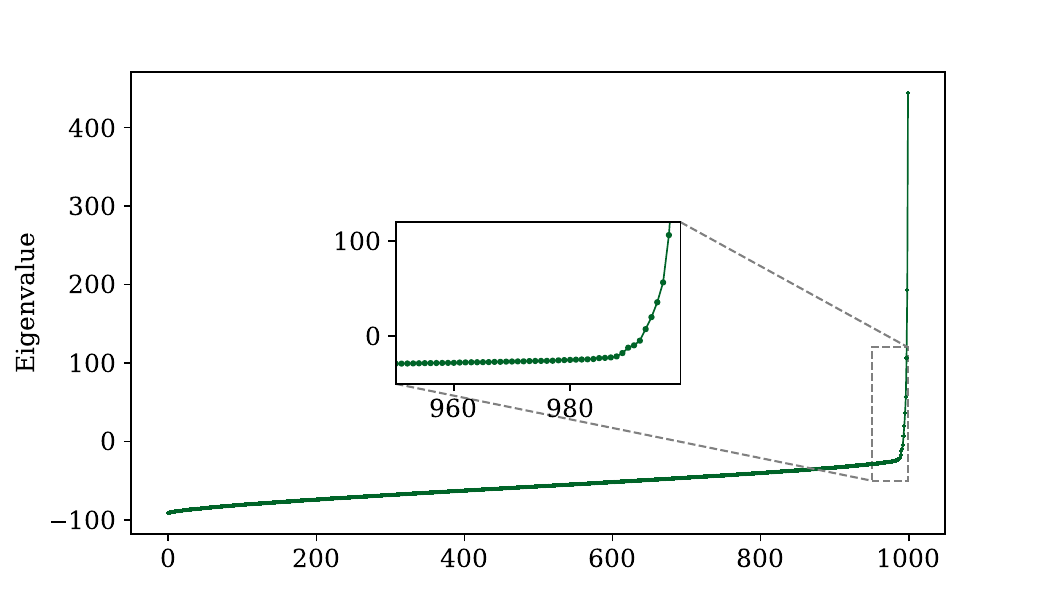}
    \caption{Empirical spectrum of $\bT$ eq. \eqref{eq:def:spectral_method_matrix} with preprocessing $\cT(y)=y/(1+|y|)$, for a hierarchical multi-index model with $g_k(z)=\frac{1}{2}\He_2(z)+\frac{1}{2\cdot4!}\He_4(z)$ and $\gamma = 1.3$. The covariates dimension is $d = 1000$, while the feature space dimension is $\twidth = 20$. The figure illustrates the change in scale of the eigenvalue gaps, transitioning from the informative spikes ($\Theta_d(1)$) to the uninformative bulk ($o_d(1)$). This behavior forms the basis of the selection method described in Def. \ref{def:spectral_estimator}.}
    \label{fig:elbow}
\end{figure}

By definition, any estimator that is a function of the dataset only has a weighted mean-squared error larger than ${\rm MMSE}_\gamma$. In this section we consider the spectral method defined in \ref{def:spectral_estimator}. The following Theorem, proven in \citet{kovacevic25aspectral}, is valid for generic Gaussian multi-index models with $\twidth$ indices. We refer to the original work for further details. In the rest of the section we denote by $\E$ the expected value with respect to $\bz\sim\cN(\bzero_{\twidth},\bI_{\twidth})$ and $Y\sim\cN\left(\sum_{k=1}^{\twidth}a_k^\star g_k(z_k),\,\Delta\right)$.

\begin{theorem}[Theorem 4.1 in \citealt{kovacevic25aspectral}] \label{app:thm:kovacevic_spectral}
Let $\cT:\R\to\R$ be a preprocessing function subject to Assumption \ref{assumption:preprocessing} and $\bT$ defined as in eq. \ref{eq:def:spectral_method_matrix}. Let $t_1\geq t_2\geq\ldots\geq t_r\geq \tau$, for some $r\in[m_\star]$, be all the solutions to
\begin{equation}\label{app:eq:kovacevic_self_consistent}
    \operatorname{det}\left(\alpha \E\left[\frac{(\bz\bz^\top-\bI_{m_\star})\cT(Y)}{t-\cT(Y)}\right] - \bI_{m_\star}\right) = 0
\end{equation}
such that 
\begin{equation}
    t_k \geq \,\overline{t}_\alpha := \,\argmin_{t\geq \tau}\; \zeta_\alpha(t),\quad \forall k\in[r],
\end{equation}
where 
\begin{equation}\label{app:eq:zeta_fun}
\zeta_\alpha(t):=t\left(1+\alpha\E\left[\frac{\cT(Y)}{t-\cT(Y)}\right]\right).
\end{equation}

    Then, denote $\lambda_1^{\bT},\ldots,\lambda^{\bT}_{\twidth}$ the largest $\twidth$ eigenvalues of $\bT$. For the top $r$ eigenvalues it holds that 
    \begin{equation}
        \lambda_1^{\bT},\ldots,\lambda^{\bT}_r \xrightarrow{\rm a.s.} \zeta_\alpha(t_1),\ldots,\zeta_\alpha(t_r),
    \end{equation}
    and for the remaining $\twidth-r$ eigenvalues it holds that 
    \begin{equation}
        \lambda_{r+1}^{\bT},\ldots,\lambda^{\bT}_{\twidth} \xrightarrow{\rm a.s.} \zeta_\alpha(\overline{t}_\alpha).
    \end{equation}
\end{theorem}
In particular, the top $r$ eigenvalues of $\bT$ are separated from each other and from the remaining ones by gaps of order one, while the gaps between the following $m_\star-r$ eigenvalues vanish as $d\to\infty$. Therefore, for $d$ large enough, the eigenvectors selected by the spectral estimator of Definition \ref{def:spectral_estimator} are those associated with the $r$ outliers of Theorem \ref{app:thm:kovacevic_spectral}.

As a first result, we derive a bound for $\alpha_k$.

Consider the matrix in eq. \eqref{app:eq:kovacevic_self_consistent}
\begin{equation}
    \E\left[\frac{(\bz\bz^\top-\bI_{m_\star})\cT(Y)}{t-\cT(Y)}\right] = \E_Y\left[\bG(Y)\frac{\cT(Y)}{t-\cT(Y)}\right],
\end{equation}
with $\bG(y) = \E[\bz\bz^\top-\bI_{m_\star}|Y=y]$.
It is straightforward to show that it is a diagonal matrix, due to the parity of each function $g_k$. Indeed, given $k\neq h$,
\begin{align}
   G_{kh}(y) =\E[z_kz_h|y] \propto \int e^{-\lVert\bz\rVert^2/2}\exp\left(-\left(y-\sum_{k'=1}^{m_\star} a_{k'} g_{k'}(z_{k'})\right)^2/(2\Delta)\right) z_h z_k \,\de \bz= 0.
\end{align}
For simplicity, we denote $G_k(y):=G_{kk}(y)$. Therefore, the solutions of eq. \eqref{app:eq:kovacevic_self_consistent} coincides with the solutions of 
\begin{equation}\label{app:eq:spectral_eigenvalue_self_cons}
    \alpha^{-1} = \E\left[\frac{\cT(Y)}{t-\cT(Y)}G_{k}(Y)\right], \quad k\in[m_\star].
\end{equation}
Further, note that, by definition $\overline{t}_\alpha$ is the solution of
\begin{equation}
    \alpha^{-1} = \E\left[\left(\frac{\cT(Y)}{t-\cT(Y)}\right)^2\right].
\end{equation}
By Theorem \ref{app:thm:kovacevic_spectral} and Theorem 4.2 in \citet{kovacevic25aspectral}, the spectral estimator has an eigenvector with a finite overlap with $\bw_k^\star$ if and only if eq.~\eqref{app:eq:spectral_eigenvalue_self_cons} has a solution $t_k>\overline t_\alpha$. For $k$ large enough, the right-hand side of eq.~\eqref{app:eq:spectral_eigenvalue_self_cons} is positive and decreasing in $t$ for $t\geq\overline t_\alpha$. This follows from the expansion of the first line of eq.~\eqref{app:eq:spectral_transition_system_approx} below, which holds for any $t$ large enough, applied both to it and to its derivative $-\E[\cT(Y)G_k(Y)/(t-\cT(Y))^2]$, so that such a solution exists if and only if
\begin{equation}
    u_k(\alpha):=\alpha\,\E\left[\frac{\cT(Y)}{\overline t_\alpha-\cT(Y)}G_k(Y)\right]>1.
\end{equation}
The $k$-critical threshold of the spectral estimator, Definition \ref{def:k_critical}, is therefore $\alpha_k^{\rm sp}=\inf\{\alpha>0:\,u_k(\alpha)>1\}$. By continuity of $u_k$, $\alpha_k^{\rm sp}$ solves $u_k(\alpha)=1$, namely, by definition of $\overline t_\alpha$, the system
\begin{equation}\label{ap:eq:system_spectral_critical}
    \begin{cases}
        (\alpha^{\rm sp}_k)^{-1} = \E_Y\left[\frac{\cT(Y)}{t-\cT(Y)}G_{k}(Y)\right]\\
        (\alpha_k^{\rm sp})^{-1} = \E_Y\left[\left(\frac{\cT(Y)}{t-\cT(Y)}\right)^2\right].
    \end{cases}
\end{equation}

In order to characterize the threshold, we consider the following expansions for small SNR $a_k^\star\ll1$. Define $\sZ$ the marginal distribution of $Y$ 
\begin{align}
    \sZ(y) &:= \frac{1}{\sqrt{2\pi\Delta}}\E_{\bz}\left[\exp\left(-\left(y-\sum_{k=1}^{\twidth}a_k^\star g_k(z_k)\right)^2/(2\Delta)\right)\right].\\
\end{align}
Note that $\sZ\in\mathcal C^\infty$ and, applying the definition of the probabilist's Hermite polynomials, its $j^{\rm th}$ derivative reads
\begin{align}\label{app:eq:derivative_Z_spectral}
    \frac{\de^j}{\de y^j}\sZ(y) = \frac{(-1)^j}{\sqrt{2\pi\Delta^{j+1}}}\E_{\bz}\left[e^{-\left(y-\sum_{k=1}^{\twidth}a_k^\star g_k(z_k)\right)^2/(2\Delta)}\He_k\left(\frac{y-\sum_{k=1}^{\twidth}a_k^\star g_k(z_k)}{\sqrt{\Delta}}\right)\right].
\end{align}
We now aim to approximate integrals of the type $\E_Y[h(Y)G_k(Y)]$, where $h:\R\to\R$ is bounded, that are involved in eq. \eqref{app:eq:kovacevic_self_consistent}. In particular, we derive a bound for 
\begin{equation}
    \mathcal E_k :=  \E[h(Y)G_k(Y)] + a_k^\star\E_z[g_k(z)(z^2-1)]\int\de yh(y)\sZ'(y).
\end{equation}
In the following, we leverage the Taylor's expansion of the complex exponential, ensuring that, for all $\omega\in\R$, $k\in\mathbb N$, there exists a constant $c\in[0,\omega]$ such that\footnote{The following expression is obtained as the integral form of the remainder in Taylor's theorem applied to the real functions $\sin(\omega)$ and $\cos(\omega)$ in $e^{i\omega} = \cos(\omega)+i\sin(\omega)$.}
\begin{align}\label{app:eq:remainder}
    R_k(\omega)&:= e^{i \omega}-\sum_{j=0}^k\frac{(i \omega)^j}{j!},\qquad R_k(\omega) = \frac{(i\omega)^{k+1}}{k!}\int_0^1 e^{i\omega t}(1-t)^k\de t
\end{align}
In addition, we use the Fourier representation of the Gaussian density $(2\pi)^{-1/2}e^{-x^2/2}=(2\pi)^{-1}\int\de\omega e^{i\omega x -\omega^2/2}$ and the identity
\begin{equation}\label{app:eq:fourier_derivative}
    \frac{1}{2\pi}\int \de \omega e^{i\omega x -\omega^2}(i\omega)^j = \frac{1}{2\pi}\frac{\de^j}{\de x^j}\int \de \omega e^{i\omega x -\omega^2} = (-1)^j\frac{e^{-x^2/2}}{\sqrt{2\pi}}\He_j(x).
\end{equation}
Define the shorthand $\rho_k^{jj'} := \E_z[g_k^{j}(z)\He_{j'}(z)]$, for $j,j'\in\{0,1,2\}$, which are non-zero and finite by Assumption \ref{assumption:gen_exp}. In order to simplify the following equations, let $\E_{\bz_{\setminus k}}$ denote the expected value with respect to $\{z_h\}_{h\neq k}$ and $S_k(\bz_{\setminus k}):= \sum_{h\neq k}a_h^\star g_h(z_h)$. Then, we have
\begin{align}
     \mathcal E_k =& \int \de y h(y)\E_{\bz} \left[\mathsf P(y|\bz)(z_k^2-1)+a_k^\star\rho_k^{12} \partial_y\sP(y|\bz)\right]\\
    \overset{\text{\eqref{app:eq:fourier_derivative}}}{=}&\E_{\bz_{\setminus k}}\frac{1}{2\pi}\int \de y h(y)\int\de \omega e^{ i \omega (y-S_k(\bz_{\setminus k})) - \omega^2\Delta/2}  \E_z\left[ {e^{-i\omega a_k^\star g_k(z)}}(z^2-1)\right] \\
    &+ \frac{a_k^\star \rho_k^{12}}{2\pi}\E_{\bz_{\setminus k}}\int \de y h(y)\int \de \omega (i\omega)e^{ i \omega (y-S_k(\bz_{\setminus k})) - \omega^2\Delta/2} \E_z[e^{-i\omega a_k^\star g_k(z)}]\\
    =&\frac{1}{2\pi}\E_{\bz_{\setminus k}}\int \de y h(y)\int\de \omega e^{ i \omega (y-S_k(\bz_{\setminus k})) - \omega^2\Delta/2}\nonumber\\
    &\quad\times\underbrace{\E_z\left[ (1 - i\omega a_k^\star g_k(z) + R_1(\omega a_k^\star g_k(z)))(z^2-1)\right] }_{0 -i\omega a_k^\star \rho_k^{12} + \E_z[R_1(\omega a_k^\star g_k(z)))(z^2-1)]}\\
    &+ \frac{a_k^\star \rho_k^{12}}{2\pi}\E_{\bz_{\setminus k}}\int \de y h(y)\int \de \omega (i\omega)e^{ i \omega (y-S_k(\bz_{\setminus k})) - \omega^2\Delta/2} \E_z[1 + R_0(\omega a_k^\star g_k(z)))].
\end{align}
After the cancellation of the terms in $a_k^\star$, substituting the integral form of the remainder in eq. \eqref{app:eq:remainder}
\begin{align}
      \mathcal E_k=&\frac{(a_k^\star)^2}{2\pi}\E_{\bz}g_k(z_k)(z_k^2-1)\int_0^1 \de t (1-t)\nonumber\\
    &\quad\times\int \de y h(y)\int\de \omega (i\omega)^2e^{ i \omega (y-S_k(\bz_{\setminus k})) - \omega^2\Delta/2}  e^{i t\omega a_k^\star g_k(z_k)} \\
    &+ \frac{(a_k^\star)^2 \rho_k^{12}}{2\pi}\E_{\bz}g_k(z_k) \int_0^1 \de t\nonumber\\
    &\quad\times\int\de y h(y)\int \de \omega (i\omega)^2e^{ i \omega (y-S_k(\bz_{\setminus k})) - \omega^2\Delta/2} e^{i t\omega a_k^\star g_k(z_k)} \\
    \overset{\text{\eqref{app:eq:fourier_derivative}}}{=}& \frac{(a_k^\star)^2}{\Delta\sqrt{2\pi\Delta}}\E_{\bz}g_k^2(z_k)(z_k^2-1)\int_0^1 \de t (1-t)\nonumber\\
    &\quad\times\int \de y h(y) e^{-(y-S_k(\bz_{\setminus k}) + t a_k^\star g_k(z_k))^2/(2\Delta)}\He_2\left(\frac{y-S_k(\bz_{\setminus k}) + t a_k^\star g_k(z_k)}{\sqrt{\Delta}}\right)\\
    &+ \frac{(a_k^\star)^2 \rho_k^{12}}{\Delta\sqrt{2\pi\Delta}}\E_{\bz}g_k(z_k) \int_0^1 \de t\nonumber\\
    &\quad\times\int\de y h(y)e^{-(y-S_k(\bz_{\setminus k}) + t a_k^\star g_k(z_k))^2/(2\Delta)}\He_2\left(\frac{y-S_k(\bz_{\setminus k}) + t a_k^\star g_k(z_k)}{\sqrt{\Delta}}\right)\\
    =& \frac{(a_k^\star)^2}{\Delta\sqrt{2\pi}}\E_{\bz}g_k^2(z_k)(z_k^2-1)\int_0^1 \de t (1-t)\nonumber\\
    &\quad\times\int \de y h\left(\sqrt{\Delta}y+S_k(\bz_{\setminus k}) - t a_k^\star g_k(z_k)\right) e^{-y^2/2}\He_2\left(y\right)\\
    &+ \frac{(a_k^\star)^2 \rho_k^{12}}{\Delta\sqrt{2\pi}}\E_{\bz}g_k(z_k) \int_0^1 \de t\nonumber\\
    &\quad\times\int\de y h\left(\sqrt{\Delta}y+S_k(\bz_{\setminus k}) - t a_k^\star g_k(z_k)\right) e^{-y^2/2}\He_2\left(y\right).
\end{align}
Recalling that $h$ is a bounded function, we obtain
\begin{align}
    |  \mathcal E_k| \leq&\sup_{y\in\R}|h(y)| \frac{(a_k^\star)^2}{\Delta\sqrt{2\pi}}  \int \de y e^{-y^2/2}|\He_2(y)|\left(\frac{\rho_k^{22}}{2} + \rho_k^{12}\rho_k^{10}\right).
\end{align}
Note that, by Assumption \ref{assumption:gen_exp}, the constants $\rho_k^{jj'}$ do not scale with $k$, therefore $|  \mathcal E_k| = O((a_k^\star)^2)$.
We now look for $t>\tau$ and $\alpha$ that are solutions of eqs. \eqref{ap:eq:system_spectral_critical}. In particular, the oracle $k$-critical threshold in Appendix \ref{app:sec:lower-bound} is a lower bound to the spectral one, therefore $\alpha_k^{\rm sp} = \Omega(\alpha_k^{\rm oracle})$ and diverges with $k$. The second equation of \eqref{ap:eq:system_spectral_critical}, due to the monotonicity of the RHS, implies that also the solution $t$ must diverge with $k$. Using the bound we have derived for $\mathcal E_k$,
\begin{equation}\label{app:eq:spectral_transition_system_approx}
    \begin{cases}
        (\alpha^{\rm sp}_k)^{-1} = -t^{-1}a_k^\star \E_z[g_k(z)(z^2-1)]\int\de y\cT(y)\sZ'(y) + O(a_k^\star t^{-2} + (a_k^\star)^2t^{-1})\\
        (\alpha_k^{\rm sp})^{-1} = t^{-2}\E[\cT(y)^2] + O(t^{-3}).
    \end{cases}
\end{equation}
The leading term on the right-hand side of the first equation is positive, so that the informative outliers emerge from the right edge of the bulk. Indeed, $\int\de y\,\cT(y)\sZ'(y) = -\int \sZ(y)\,\de\cT(y) < 0$, since $\cT$ is strictly increasing, and $\E_z[g_k(z)(z^2-1)] > 0$, since $g_k$ is even and strictly increasing on $(0,\infty)$. Note that this is the only place where the monotonicity of $g_k$ enters.
In the previous computations, we have used 
\begin{align}
\left|\int\de y\left(\frac{\cT(y)\sZ'(y)}{t-\cT(y)} - t^{-1}\cT(y)\sZ'(y)\right)\right| &= t^{-2}\left|\int\de y\frac{\cT^2(y)\sZ'(y)}{1-\cT(y)/t}\right| \\&\leq t^{-2} \sup_{y\in\R}\left(\frac{\cT^2(y)}{1-\cT(y)/t}\right)\int\de y|\sZ'(y)| \\
&\overset{\eqref{app:eq:derivative_Z_spectral}}{=} t^{-2} \sup_{y\in\R}\left(\frac{\cT^2(y)}{1-\cT(y)/t}\right) \int \de y \frac{e^{-y^2/2}}{\sqrt{2\pi\Delta}}|y|= O(t^{-2}),
\end{align}
where, for $t$ sufficiently large, the supremum is bounded by a constant independent of $t$.
The system \eqref{app:eq:spectral_transition_system_approx} is thereby solved for 
\begin{equation}
    t^{-1} =\Theta(a_k^\star),\qquad \alpha_k^{\rm sp} = \Theta(\alpha_k^{\rm oracle}).
\end{equation}
Moreover, for $\alpha\gg1$, the second line of eq.~\eqref{app:eq:spectral_transition_system_approx} gives $\overline t_\alpha=\Theta(\sqrt\alpha)$, and the first one gives $u_k(\alpha)\geq c\,a_k^\star\sqrt\alpha$ for $k$ large enough, with $c>0$ independent of $k$ and $\alpha$. Therefore, as $u_k(\alpha)>1$ for all $\alpha>(c\,a_k^\star)^{-2}=\Theta(\alpha_k^{\rm sp})$, once direction $k$ is weakly recovered, it remains recovered at larger sample complexities.

\begin{remark}[Outliers at the left edge of the bulk]\label{app:remark:left_edge}
If $\E_z[g_k(z)(z^2-1)]<0$, the leading term in the first line of eq.~\eqref{app:eq:spectral_transition_system_approx} is negative, and eq.~\eqref{app:eq:spectral_eigenvalue_self_cons} has no solution $t_k>\overline t_\alpha$, which implies that direction $k$ produces no outlier at the right edge of the bulk. Since the smallest eigenvalues of $\bT$ are the largest eigenvalues of $-\bT$, i.e. of the matrix in eq.~\eqref{eq:def:spectral_method_matrix} with preprocessing $-\cT$, applying Theorem~\ref{app:thm:kovacevic_spectral} to $-\cT$ and repeating the computation above, the leading term in the first line of eq.~\eqref{app:eq:spectral_transition_system_approx} changes sign, while the second line is unchanged. Direction $k$ then produces an outlier at the left edge of the bulk, with the same scaling of the threshold and of the overlap. Hence, without the monotonicity of $g_k$, Theorem~\ref{thm:spectral_rates} holds for the spectral estimator that retains the outliers at both edges of the bulk.
\end{remark}

We can now derive the scaling laws for the spectral weighted mean-squared error. Taking $\alpha\gg\alpha_k^{\rm sp}$, for $k$ large enough, the solution $t_k$ of eq. \eqref{app:eq:spectral_eigenvalue_self_cons} scales as $t_k = \Theta(\alpha a_k^\star)$. As a consequence of Theorem 4.2 in \citet{kovacevic25aspectral}, specialized in our setting -- where we can exploit the diagonality of $\bG(y)$, each solution of eq. \eqref{app:eq:kovacevic_self_consistent} is associated with a single direction, and we denote by $t_h$ the solution associated with $\bw_h^\star$. Let $\pi:[r]\to[\twidth]$ be the injective map such that the $k^{\rm th}$ largest eigenvalue $\lambda_k^{\bT}$ converges to $\zeta_\alpha(t_{\pi(k)})$. Then, the overlap between the eigenvector $\bv_k$ corresponding to the eigenvalue $\lambda_k^{\bT}$, and the true signal $\bw_h^\star$, converges, in the high-dimensional limit, to an overlap
\begin{equation}
    m_{hk}^2 := \frac{1}{d^2}\langle\bw_h^\star,\bv_k\rangle^2 = \delta_{h\pi(k)} \frac{\zeta'_\alpha(t_h)}{\zeta'_\alpha(t_h) - \frac{\de}{\de t}R_h(t_h)},
\end{equation}
where $R_h(t) = \alpha t\,\E[z_h^2 \cT(y)/(t-\cT(y))]$. For $\alpha\gg\alpha_k^{\rm sp}$ we have that $t_k\gg 1$ and eq. \eqref{app:eq:spectral_eigenvalue_self_cons} is solved by $t_k = \Theta(a_k^\star \alpha)$. Putting all together
\begin{equation}
  m_{hk}^2  = \delta_{h\pi(k)}\left(1 - \Theta\left(\frac{(a_h^\star)^{-2}}{\alpha}\right) \right).
\end{equation}
Recall that
\begin{equation}
    {\rm mse}_h(\bv_k) = \frac{1}{d^2}\E\left[\|\bv_k\bv_k^\top\|^2_{\rm F}  + \|\bw_h^\star\bw_h^{\star\top}\|^2_{\rm F} -2 \langle\bw_h^\star,\bv_k\rangle^2\right] \to 2(1-m_{hk}^2),
\end{equation}
which tends to $2$ for $h\neq\pi(k)$, while ${\rm mse}_h(\bzero)=1$. Therefore, the minimum over permutations in Definition \ref{def:weighted_MSE} is attained, up to constant factors, by pairing each $\bv_k$ with $\bw_{\pi(k)}^\star$ and the zero rows with the remaining directions.
By repeating computations analogous to the ones in Appendix \ref{app:sec:lower-bound}, one finds, setting $\hat{\bW} = (\bv_1,\ldots,\bv_r)^\top$,
\begin{equation}
    {\rm MSE}_\gamma(\hat{\bW} := (\hat{\bW}^{\rm sp},\bzero_{(\twidth-r)\times d})) = \Theta({\rm MMSE}_\gamma^{\rm oracle}),
\end{equation}
which proves Theorem \ref{theorem:bayes_rates} and \ref{thm:spectral_rates}.

Figure~\ref{fig:plateaus} illustrates the sequential recovery underlying Theorem~\ref{thm:spectral_rates}. The features are recovered one at a time, each at its own threshold $\alpha_k^{\rm sp}$, where its error ${\rm mse}_k$ drops. 
\begin{figure}[ht]
    \centering
    \includegraphics[width=\linewidth]{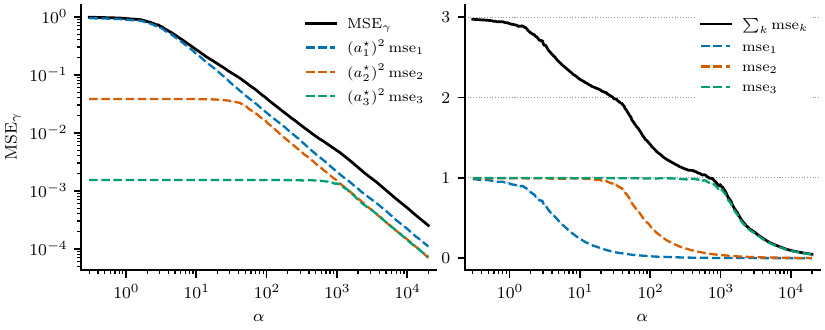}
    \caption{Illustration of the sequential recovery of the features by the spectral estimator of Definition~\ref{def:spectral_estimator}, with $\cT(y)=y/(1+|y|)$. The target is a hierarchical multi-index model with $\twidth=3$, $g_k(z)=|z|-\sqrt{2/\pi}$ for all $k$, $a^\star\propto(1, 0.2, 0.04)$ with $\sum_k (a_k^\star)^2=1$, and $\Delta=0.5$; $d=300$. ({\bf left}) ${\rm MSE}_\gamma$ and its contributions $(a_k^\star)^2{\rm mse}_k$ (dashed). ({\bf right}) Unweighted error $\sum_k {\rm mse}_k$ and the individual ${\rm mse}_k$ (dashed); horizontal dotted lines mark the plateau values of $\sum_k{\rm mse}_k$ when $j$ features are not yet recovered, $j\in[m_\star]$. Averages over $20$ instances; shaded bands represents one standard error.}
    \label{fig:plateaus}
\end{figure}

\section{Proof of Theorem \ref{thm:readout}}\label{app:sec:readout}

In this section we prove Theorem \ref{thm:readout}, stating that a two-layer neural network, trained according to Algorithm \ref{alg:network_training}, can achieve an excess risk bottlenecked by the information-theoretic error on the recovery of the representations $\bW_\star$. Throughout, we work conditionally on the dataset $\cD_1$ used to compute the spectral estimator $\hat{\bW}^{\rm sp}\in\R^{r\times d}$ of Definition \ref{def:spectral_estimator}. To lighten the notation, in this section we denote by $n$ the number of samples in $\cD_2$, i.e. $n/2$ in Algorithm \ref{alg:network_training}. The activation $\sigma$ is assumed to be bounded, $|\sigma|\leq B$, continuous and nonconstant. Moreover, we assume that each link function $g_k$, $k\in[\twidth]$, is weakly differentiable with $\E_z[g_k'(z)^2]\leq C'$, for a constant $C'$ independent of $k$. This holds, for instance, for Lipschitz link functions and for polynomials of bounded degree.

\paragraph{Reduction to the projected inputs.} Denote by $\hat{\bw}_j$ the $j^{\rm th}$ row of $\hat{\bW}^{\rm sp}$. Given a covariate $\bx\in\R^d$, define the projected inputs $\bs=\hat{\bW}^{\rm sp}\bx\in\R^r$ and the target indices $\bz^\star=\bW_\star\bx\in\R^{\twidth}$. The first layer pre-activations are given by
\begin{equation}
    \bW\bx = \bZ\hat{\bW}^{\rm sp}\bx = \bZ\bs,
\end{equation}
so that the network with second-layer weights $\bc\in\R^p$, $f(\bx;\bTheta)=\bc^\top\sigma(\bZ\bs+\bb)/\sqrt{p}$, depends on $\bx$ only through $\bs$. The rows of $\hat{\bW}^{\rm sp}$ are eigenvectors of the symmetric matrix $\bT$, normalized as $\|\hat{\bw}_j\|^2=d$, hence they are orthogonal. Moreover, $\hat{\bW}^{\rm sp}$ is independent of $\cD_2$ and of the test point. Therefore, for $\bx\sim\cN(\bzero,d^{-1}\bI_d)$, the pair $(\bs,\bz^\star)$ is jointly Gaussian with
\begin{equation}
    \bs\sim\cN(\bzero,\bI_r),\qquad \bz^\star\sim\cN(\bzero,\bI_{\twidth}),\qquad \E[\bs\bz^{\star\top}]=\bM_d:=\frac{1}{d}\hat{\bW}^{\rm sp}\bW_\star^\top\in\R^{r\times\twidth},
\end{equation}
and the joint law of $(\bs,y)$ depends on $d$ only through the overlap matrix $\bM_d$. Throughout this section, the subscript $d$ denotes a quantity at finite size, which depends on $\cD_1$ through $\bM_d$, and its absence the corresponding deterministic limit as $n,d\to\infty$. In Appendix \ref{app:sec:spectral} we have shown that, as $n,d\to\infty$, each row $j$ of the spectral estimator has asymptotic overlap $m_{\pi(j)}$ with a single signal component $\bw_{\pi(j)}^\star$, and vanishing overlap with the others, where $\pi:[r]\to[\twidth]$ is injective. Since flipping the sign of $\hat{\bw}_j$ amounts to flipping the $j$-th column of $\bZ$, whose law is symmetric, we fix the sign of each $\hat{\bw}_j$ so that its overlap with $\bw_{\pi(j)}^\star$ is nonnegative. Then $\bM_d$ converges in probability to $\bM$, with $M_{jh}=m_{\pi(j)}\delta_{h\pi(j)}$. We denote by $\cK:=\pi([r])$ the set of recovered directions. Its size $r$ is the number of outliers of $\bT$,\footnote{More precisely, $r$ equals the number of outliers with probability tending to one as $n,d\to\infty$.} so that $r\asymp k_\alpha$, with $k_\alpha$ defined in Appendix \ref{app:sec:lower-bound}.

With respect to the projected input space, the effective target function is
\begin{equation}
    f_{\star,d}^{\rm eff}(\bs) := \E[f_\star(\bx)|\bs] = \E[y|\bs].
\end{equation}
We denote by $\|h\|_{L^2}^2:=\E_{\bs\sim\cN(\bzero,\bI_r)}[h(\bs)^2]$ the squared norm of $L^2(\cN(\bzero,\bI_r))$. The effective target depends on $d$ only through $\bM_d$, and we denote by $f_\star^{\rm eff}$ the effective target computed at the limit, i.e. with $\E[\bs\bz^{\star\top}]=\bM$. Below, we replace $f_{\star,d}^{\rm eff}$ by $f_\star^{\rm eff}$, both in the irreducible error and in the choice of the second-layer weights. This is allowed by the following lemma, proven at the end of Appendix \ref{app:sec:readout}.
\begin{lemma}\label{app:lemma:continuity}
    As $n,d\to\infty$, $\|f_{\star,d}^{\rm eff}-f_\star^{\rm eff}\|_{L^2}$ converges to zero in probability.
\end{lemma}

For any network whose weights are independent of the test point, $f(\bx;\bTheta)$ is a function of $\bs$, while $f_{\star,d}^{\rm eff}(\bs)$ is the conditional expectation of $f_\star(\bx)$ given $\bs$. The cross term then vanishes, and the excess risk eq.~\eqref{eq:def:risk} decomposes as
\begin{equation}\label{app:eq:risk_decomposition}
    R(\bTheta) = \underbrace{\E\left[\left(f_\star(\bx)-f_{\star,d}^{\rm eff}(\bs)\right)^2\right]}_{=:R_{\star,d}} + \underbrace{\left\|f_{\star,d}^{\rm eff}-p^{-1/2}\bc^\top\sigma(\bZ\,\cdot\,+\bb)\right\|_{L^2}^2}_{=:R_{\rm ro}(\bc)}
\end{equation}
where $\bc\in\R^p$ denotes the second-layer weights of the network.
The first term is the irreducible error due to the imperfect recovery of the representations, and does not depend on the second layer. The second term is the readout error. We now show that the former scales as ${\rm MMSE}_\gamma$, and that the latter can be made smaller than the former by taking the width $p$ large enough and the regularization $\lambda$ small enough.

\paragraph{Irreducible error.} We expand each link function in Hermite polynomials,
\begin{equation}
    g_k(z)=\sum_{\ell\geq0}\frac{g_{k,\ell}}{\ell!}\He_\ell(z),\qquad g_{k,\ell}:=\E_z[g_k(z)\He_\ell(z)],
\end{equation}
where the series converges in $L^2(\cN(0,1))$ and $\E_z[g_k(z)^2]=\sum_{\ell}g_{k,\ell}^2/\ell!$. Since $\He_\ell'=\ell\He_{\ell-1}$, by Stein's lemma $\E_z[g_k'(z)\He_{\ell-1}(z)]=g_{k,\ell}$, hence $\E_z[g_k'(z)^2]=\sum_{\ell\geq1}\ell\,g_{k,\ell}^2/\ell!$. If $(u,v)$ are standard Gaussian variables with correlation $m$, then $\E[\He_\ell(u)\He_{\ell'}(v)]=\delta_{\ell\ell'}\,\ell!\,m^\ell$ \citep{ODonnell_2014}. Since $u|v\sim\cN(m v,1-m^2)$, $\E[\He_\ell(u)|v]$ is a polynomial of degree at most $\ell$ in $v$, and the previous identity gives its expansion in Hermite polynomials,
\begin{equation}\label{app:eq:mehler}
    \E[\He_\ell(u)|v]=m^\ell\He_\ell(v).
\end{equation}
At the limit $\bM$, for $h=\pi(j)\in\cK$ the index $z_h^\star$ is correlated with $s_j$ only, with correlation $m_{\pi(j)}$, while for $h\notin\cK$ it is independent of $\bs$. By eq.~\eqref{app:eq:mehler},
\begin{equation}
    f_\star^{\rm eff}(\bs)=\sum_{j=1}^r a_{\pi(j)}^\star\sum_{\ell\geq0}\frac{g_{\pi(j),\ell}}{\ell!}m_{\pi(j)}^\ell\He_\ell(s_j)+\sum_{h\notin\cK}a_h^\star g_{h,0}.
\end{equation}
As $f_{\star,d}^{\rm eff}$ is a conditional expectation, $R_{\star,d}=\E[f_\star(\bx)^2]-\|f_{\star,d}^{\rm eff}\|_{L^2}^2$. We denote by $R_\star:=\E[f_\star(\bx)^2]-\|f_\star^{\rm eff}\|_{L^2}^2$ its value at the limit $\bM$. Using the independence of the $z^\star_h$, the independence of the $s_j$ and the orthogonality of Hermite polynomials, we find
\begin{equation}
    R_\star=\sum_{h\in\cK}(a_h^\star)^2\sum_{\ell\geq1}\frac{g_{h,\ell}^2}{\ell!}\left(1-m_h^{2\ell}\right)+\sum_{h\notin\cK}(a_h^\star)^2\sum_{\ell\geq1}\frac{g_{h,\ell}^2}{\ell!}.
\end{equation}
For $m\in[0,1]$ and $\ell\geq1$, $1-m^2\leq1-m^{2\ell}\leq\ell(1-m^2)$. Moreover, by Assumption \ref{assumption:gen_exp}, \[\sum_{\ell\geq1}g_{h,\ell}^2/\ell!\geq g_{h,2}^2/2=\E_z[g_h(z)(z^2-1)]^2/2>D^2/2\] and \[\sum_{\ell\geq1}\ell\,g_{h,\ell}^2/\ell!=\E_z[g_h'(z)^2]\leq C',\]
which implies
\begin{align*}
R_\star &\geq \sum_{h\in\cK}(a_h^\star)^2\left(1-m_h^{2}\right)\sum_{\ell\geq1}\frac{g_{h,\ell}^2}{\ell!}+\sum_{h\notin\cK}(a_h^\star)^2\sum_{\ell\geq1}\frac{g_{h,\ell}^2}{\ell!}\\
&\geq \frac{D^2}{2}\left(\sum_{h\in\cK}(a_h^\star)^2\left(1-m_h^{2}\right)+\sum_{h\notin\cK}(a_h^\star)^2\right)
\end{align*}
and, denoting $C'' := \max\{\E_z[g_h(z)^2]:h\in[m_\star]\}$, 
\begin{align*}
R_\star &\leq \sum_{h\in\cK}(a_h^\star)^2\left(1-m_h^{2}\right)\sum_{\ell\geq1}\frac{\ell g_{h,\ell}^2}{\ell!}+\sum_{h\notin\cK}(a_h^\star)^2\sum_{\ell\geq1}\frac{g_{h,\ell}^2}{\ell!}\\
&\leq C'\sum_{h\in\cK}(a_h^\star)^2\left(1-m_h^{2}\right)+C''\sum_{h\notin\cK}(a_h^\star)^2\\
&\leq \max(C',C'')\left(\sum_{h\in\cK}(a_h^\star)^2\left(1-m_h^{2}\right)+\sum_{h\notin\cK}(a_h^\star)^2\right).
\end{align*}
Therefore,
\begin{equation}
    R_\star=\Theta\left(\sum_{h\in\cK}(a_h^\star)^2(1-m_h^2)+\sum_{h\notin\cK}(a_h^\star)^2\right)=\Theta({\rm MMSE}_\gamma),
\end{equation}
where the last equality follows from Theorem \ref{thm:spectral_rates}, since, as shown at the end of Appendix \ref{app:sec:spectral}, the weighted mean-squared error of the zero-padded spectral estimator is of the same order as the sum in the middle term. As for the spectral method mean-squared error, the two terms correspond to the approximation error of learned features and underfitting of unlearned ones. Finally, $\E[f_\star(\bx)^2]$ does not depend on $\bM_d$, while $\|f_{\star,d}^{\rm eff}\|_{L^2}^2$ converges in probability to $\|f_\star^{\rm eff}\|_{L^2}^2$ by Lemma \ref{app:lemma:continuity}, so that $R_{\star,d}$ converges in probability to $R_\star$ as $n,d\to\infty$.

\paragraph{Readout error.} Define the normalized features $\bphi(\bs):=\sigma(\bZ\bs+\bb)/\sqrt{p}\in\R^p$, and, for $\bc\in\R^p$, denote by $f_{\bc}(\bs):=\bc^\top\bphi(\bs)$ the predictor with second-layer weights $\bc$, so that $R_{\rm ro}(\bc)=\|f_{\star,d}^{\rm eff}-f_{\bc}\|_{L^2}^2$. Since $\|\bphi(\bs)\|^2\leq B^2$ for all $\bs$, independently of $p$, the Cauchy--Schwarz inequality gives
\begin{equation}\label{app:eq:predictor_bound}
    f_{\bc}(\bs)^2\leq B^2\|\bc\|^2\qquad\text{for all }\bs\in\R^r.
\end{equation}
Given $\cD_2=\{(\bx_i,y_i)\}_{i\in[n]}$ and $\bs_i=\hat{\bW}^{\rm sp}\bx_i$, the second layer $\hat\ba$ returned by Algorithm \ref{alg:network_training} is the minimizer of $\frac{1}{n}\sum_{i}(y_i-f_{\bc}(\bs_i))^2+\lambda\|\bc\|^2$ over $\bc\in\R^p$. Its population counterpart is the minimizer $\ba_\lambda$ of
\begin{equation}
    J(\bc):=\E\left[(y-f_{\bc}(\bs))^2\right]+\lambda\|\bc\|^2,
\end{equation}
where the expectation is over a fresh pair $(\bs,y)$, at fixed $\bZ,\bb$. We split the readout error of the trained network into an approximation and an estimation term,
\begin{equation}\label{app:eq:readout_split}
    \E_{\cD_2}R_{\rm ro}(\hat\ba)\leq 2\underbrace{\|f_{\star,d}^{\rm eff}-f_{\ba_\lambda}\|_{L^2}^2}_{\text{approximation}}+2\underbrace{\E_{\cD_2}\|f_{\ba_\lambda}-f_{\hat\ba}\|_{L^2}^2}_{\text{estimation}},
\end{equation}
and bound the two terms separately.

\textit{Estimation term.} Setting the gradients of the two objectives to zero,
\begin{equation}\label{app:eq:zero_grad_identities}
    \left(\hat\bSigma+\lambda\bI_p\right)\hat\ba=\hat\bv,\qquad \left(\bSigma+\lambda\bI_p\right)\ba_\lambda=\bv,
\end{equation}
where
\begin{align}
    \hat\bSigma&:=\frac{1}{n}\sum_{i=1}^{n}\bphi(\bs_i)\bphi(\bs_i)^\top, & \hat\bv&:=\frac{1}{n}\sum_{i=1}^{n}\bphi(\bs_i)\,y_i,\\
    \bSigma&:=\E\left[\bphi(\bs)\bphi(\bs)^\top\right], & \bv&:=\E\left[\bphi(\bs)\,y\right].
\end{align}
Subtracting $(\hat\bSigma+\lambda\bI_p)\ba_\lambda$ from the first identity in (\ref{app:eq:zero_grad_identities}) and using the second one,
\begin{equation}
    \left(\hat\bSigma+\lambda\bI_p\right)(\hat\ba-\ba_\lambda)=\hat\bv-\hat\bSigma\ba_\lambda-\left(\bv-\bSigma\ba_\lambda\right)=\frac{1}{n}\sum_{i=1}^{n}\boldsymbol{u}_i-\E[\boldsymbol{u}],
\end{equation}
where $\boldsymbol{u}_i:=\bphi(\bs_i)\left(y_i-f_{\ba_\lambda}(\bs_i)\right)$ are i.i.d. copies of $\boldsymbol{u}:=\bphi(\bs)\left(y-f_{\ba_\lambda}(\bs)\right)$. Since $\hat\bSigma$ is positive semi-definite, $\|(\hat\bSigma+\lambda\bI_p)^{-1}\|_{\rm op}\leq\lambda^{-1}$. Therefore,
\begin{align*}
    \E_{\cD_2}\|f_{\ba_\lambda}-f_{\hat\ba}\|_{L^2}^2 &\leq B^2\,\E_{\cD_2}\|\hat\ba-\ba_\lambda\|^2 &&\text{(by eq.~\eqref{app:eq:predictor_bound})}\\
    &\leq \frac{B^2}{\lambda^2}\,\E_{\cD_2}\Big\|\frac{1}{n}\sum_{i=1}^{n}\boldsymbol{u}_i-\E[\boldsymbol{u}]\Big\|^2\\
    &\leq \frac{B^2}{\lambda^2n}\,\E\|\boldsymbol{u}\|^2 &&\text{(variance of an empirical mean)}\\
    &\leq \frac{B^4}{\lambda^2n}\,\E\left[\left(y-f_{\ba_\lambda}(\bs)\right)^2\right] &&\text{(since $\|\bphi(\bs)\|^2\leq B^2$)}\\
    &\leq \frac{B^4}{\lambda^2n}\,J(\ba_\lambda)\leq\frac{B^4}{\lambda^2n}\,J(\bzero) &&\text{(since $\ba_\lambda$ minimizes $J$)}.
\end{align*}
Since $J(\bzero)=\E[y^2]=:C_y=\E[f_\star(\bx)^2]+\Delta$,
\begin{equation}\label{app:eq:readout_estimation}
    \E_{\cD_2}\|f_{\ba_\lambda}-f_{\hat\ba}\|_{L^2}^2\leq\frac{B^4C_y}{\lambda^2n},
\end{equation}
where the right-hand side does not depend on $p$, $r$, $d$, $\bM_d$, $\bZ$ and $\bb$.

\textit{Approximation term.} Recalling the definition of $f_{\star,d}^{\rm eff}(\bs)=\E[y|\bs]$ and the fact that $f_{\bc}$ depends only on $\bs$, 

\begin{align*}
J(\bc)=\E[(y-f_{\star,d}^{\rm eff}(\bs))^2]+\|f_{\star,d}^{\rm eff}-f_{\bc}\|_{L^2}^2+\lambda\|\bc\|^2,
\end{align*}
where the first term does not depend on $\bc$. As $\ba_\lambda$ minimizes $J$, for any $\bc\in\R^p$, possibly depending on $\bZ$ and $\bb$, the approximation term can be bounded as follows
\begin{equation}\label{app:eq:readout_comparison}
    \|f_{\star,d}^{\rm eff}-f_{\ba_\lambda}\|_{L^2}^2\leq\|f_{\star,d}^{\rm eff}-f_{\bc}\|_{L^2}^2+\lambda\|\bc\|^2.
\end{equation}
We now construct second-layer weights $\bc$ for which the right-hand side is small enough to produce the result. The pairs $(\bz_j,b_j)$, formed by the rows of $\bZ$ and the biases, are i.i.d. copies of $(\bz,b)$, with $\bz\sim\cN(\bzero,r^{-1}\bI_r)$ and $b\sim\cN(0,1)$ independent. We denote by $L^2_{\bz,b}$ the space of functions $w:\R^{r+1}\to\R$ with $\|w\|^2_{L^2_{\bz,b}}:=\E_{\bz,b}[w(\bz,b)^2]<\infty$. For $w\in L^2_{\bz,b}$, define the infinite-width predictor
\begin{equation}
    f^\infty_w(\bs):=\E_{\bz,b}\left[w(\bz,b)\,\sigma(\bz^\top\bs+b)\right],
\end{equation}
which satisfies $|f^\infty_w(\bs)|\leq B\|w\|_{L^2_{\bz,b}}$. We construct $\bc$ in two steps: we first approximate $f_\star^{\rm eff}$ by an infinite-width predictor $f^\infty_w$, and then approximate $f^\infty_w$ by a finite-width network with the random features $(\bz_j,b_j)$, with weights of controlled norm.
The first step relies on the following density result, which adapts the argument of \citet{hornik1991approximation}.
\begin{lemma}\label{app:lemma:density}
    Let $\sigma$ be bounded, continuous and nonconstant. Then the set $\{f^\infty_w:\,w\in L^2_{\bz,b}\}$ is dense in $L^2(\cN(\bzero,\bI_r))$.
\end{lemma}
The proof is given at the end of Appendix \ref{app:sec:readout}. For the second step, given $w\in L^2_{\bz,b}$, let $\bc\in\R^p$ with $c_j:=w(\bz_j,b_j)/\sqrt{p}$. Since the pairs $(\bz_j,b_j)$ are i.i.d. copies of $(\bz,b)$ and $\E_{\bz,b}[w(\bz,b)\sigma(\bz^\top\bs+b)]=f^\infty_w(\bs)$,
\begin{align*}
    \E_{\bZ,\bb}\|f^\infty_w-f_{\bc}\|_{L^2}^2&=\E_{\bs}\,{\rm Var}_{\bZ,\bb}\Big[\frac1p\sum_{j=1}^p w(\bz_j,b_j)\sigma(\bz_j^\top\bs+b_j)\Big]\\
    &=\frac1p\,\E_{\bs}\,{\rm Var}_{\bz,b}\left[w(\bz,b)\sigma(\bz^\top\bs+b)\right]\\
    &\leq\frac1p\,\E_{\bs}\E_{\bz,b}\left[w(\bz,b)^2\sigma(\bz^\top\bs+b)^2\right]\leq\frac{B^2}{p}\E_{\bz,b}[w(\bz,b)^2].
\end{align*}
Therefore,
\begin{equation}\label{app:eq:monte_carlo}
    \E_{\bZ,\bb}\|f^\infty_w-f_{\bc}\|_{L^2}^2\leq\frac{B^2\|w\|^2_{L^2_{\bz,b}}}{p},\qquad \E_{\bZ,\bb}\|\bc\|^2=\frac{1}{p}\sum_{j=1}^p\E_{\bz,b}[w(\bz_j,b_j)^2]=\|w\|^2_{L^2_{\bz,b}}.
\end{equation}
Fix $\epsilon>0$. By Lemma \ref{app:lemma:density}, there exists $w_\epsilon\in L^2_{\bz,b}$ such that $\|f_\star^{\rm eff}-f^\infty_{w_\epsilon}\|_{L^2}\leq\epsilon$, and we take $\bc$ as above with $w=w_\epsilon$. Then, averaging eq.~\eqref{app:eq:readout_comparison} over $\bZ,\bb$ and using twice $(a+b)^2\leq2a^2+2b^2$,
\begin{align*}
    \E_{\bZ,\bb}\|f_{\star,d}^{\rm eff}-f_{\ba_\lambda}\|_{L^2}^2 &\leq 2\|f_{\star,d}^{\rm eff}-f_\star^{\rm eff}\|_{L^2}^2+4\|f_\star^{\rm eff}-f^\infty_{w_\epsilon}\|_{L^2}^2\\
    &\quad+4\,\E_{\bZ,\bb}\|f^\infty_{w_\epsilon}-f_{\bc}\|_{L^2}^2+\lambda\,\E_{\bZ,\bb}\|\bc\|^2\\
    &\leq 2\|f_{\star,d}^{\rm eff}-f_\star^{\rm eff}\|_{L^2}^2+4\epsilon^2+\left(\frac{4B^2}{p}+\lambda\right)\|w_\epsilon\|^2_{L^2_{\bz,b}} &&\text{(by eq.~\eqref{app:eq:monte_carlo})}.
\end{align*}

\paragraph{Conclusion.} Inserting the two bounds in eq.~\eqref{app:eq:readout_split}, the readout error of the trained network satisfies
\begin{equation}\label{app:eq:readout_final}
    \E_{\cD_2,\bZ,\bb}R_{\rm ro}(\hat\ba)\leq\frac{2B^4C_y}{\lambda^2n}+4\|f_{\star,d}^{\rm eff}-f_\star^{\rm eff}\|_{L^2}^2+8\epsilon^2+\left(\frac{8B^2}{p}+2\lambda\right)\|w_\epsilon\|^2_{L^2_{\bz,b}},
\end{equation}
for every $n$, $p$ and $\lambda>0$. Note that the dependence on $r$ enters only through $w_\epsilon$, and $r$ is fixed as $n,d\to\infty$. At fixed $\lambda$, as $n,d\to\infty$, the first term vanishes, whatever the dependence of $p$ on $n$, and the second converges to zero in probability, by Lemma \ref{app:lemma:continuity}. For $p\geq8B^2\|w_\epsilon\|^2_{L^2_{\bz,b}}/\epsilon^2$ and $\lambda\leq\epsilon^2/(2\|w_\epsilon\|^2_{L^2_{\bz,b}})$, the expected readout error is then at most $10\epsilon^2+o_{\mathbb P}(1)$. Take $\epsilon$ such that $\epsilon^2=R_\star/10$. By eq.~\eqref{app:eq:risk_decomposition}, for every $\alpha$ there exist a finite width $p(\alpha)$ and a regularization $\lambda(\alpha)>0$ such that, for $p\geq p(\alpha)$ and $\lambda\leq\lambda(\alpha)$,
\begin{equation}
    R_\star+o_{\mathbb P}(1)\leq R_{\star,d}+\E_{\cD_2,\bZ,\bb}R_{\rm ro}(\hat\ba)\leq 2R_\star+o_{\mathbb P}(1),
\end{equation}
with no condition relating $p$ and $n$. Taking the expectation over $\cD_1$, $R_\star+o(1)\leq R_{\rm NN}\leq 2R_\star+o(1)$. Since $R_\star=\Theta({\rm MMSE}_\gamma)$, the readout introduces no additional bottleneck, and $R_{\rm NN}=\Theta({\rm MMSE}_\gamma)$.

Moreover, eq.~\eqref{app:eq:readout_final} shows that the readout error vanishes if the width grows and the regularization decreases slowly enough. If $p_n\to\infty$ and $\lambda_n\to0$ with $n\lambda_n^2\to\infty$, the first and last terms vanish for every $\epsilon$, so that the expected readout error is at most $8\epsilon^2+o_{\mathbb P}(1)$ for any $\epsilon>0$, and $R_{\rm NN}$ converges to $R_\star$.

\paragraph{Proofs of the Lemmas \ref{app:lemma:continuity} and \ref{app:lemma:density}.}
\begin{proof}[Proof of Lemma \ref{app:lemma:continuity}]
For $\bQ\in\R^{r\times\twidth}$ with $\|\bQ\|_{\rm op}\leq1$, denote by $f_\star^{\rm eff}(\cdot;\bQ)$ the effective target computed with $\E[\bs\bz^{\star\top}]=\bQ$, so that $f_{\star,d}^{\rm eff}=f_\star^{\rm eff}(\cdot;\bM_d)$ and $f_\star^{\rm eff}=f_\star^{\rm eff}(\cdot;\bM)$. Since $\bM_d$ converges in probability to $\bM$, by the continuous mapping theorem it suffices to show that the map $\bQ\mapsto f_\star^{\rm eff}(\cdot;\bQ)$ is continuous in $L^2(\cN(\bzero,\bI_r))$.

By Assumption \ref{assumption:gen_exp}, $f_\star(\bx)=\sum_k a_k^\star g_k(z_k^\star)$ is a square-integrable function of $\bz^\star\sim\cN(\bzero,\bI_{\twidth})$. Fix $\epsilon>0$. Since polynomials are dense in $L^2(\cN(\bzero,\bI_{\twidth}))$, there exists a polynomial $q_\epsilon$ with $\E[(f_\star-q_\epsilon)^2]\leq\epsilon^2$. We denote by $q_\epsilon^{\rm eff}(\bs;\bQ):=\E[q_\epsilon(\bz^\star)|\bs]$ its conditional expectation computed with $\E[\bs\bz^{\star\top}]=\bQ$.

Since the law of $\bz^\star$ does not depend on $\bQ$, by Jensen's inequality 
\begin{equation}\label{app:eq:contraction}
    \|f_\star^{\rm eff}(\cdot;\bQ)-q_\epsilon^{\rm eff}(\cdot;\bQ)\|_{L^2}\leq\epsilon\qquad\text{for all }\bQ.
\end{equation}
Then, since $\bz^\star|\bs\sim\cN(\bQ^\top\bs,\bI_{\twidth}-\bQ^\top\bQ)$ and the moments of a Gaussian vector are polynomials in its mean and covariance, $q_\epsilon^{\rm eff}(\bs;\bQ)$ is a polynomial in $\bs$ whose coefficients are polynomials in the entries of $\bQ$. As $\bs$ has finite moments of all orders, the map $\bQ\mapsto q_\epsilon^{\rm eff}(\cdot;\bQ)$ is continuous in $L^2(\cN(\bzero,\bI_r))$. Fix $\bQ$. There exists $\delta>0$ such that
\begin{equation}\label{app:eq:poly_continuity}
    \|q_\epsilon^{\rm eff}(\cdot;\bQ')-q_\epsilon^{\rm eff}(\cdot;\bQ)\|_{L^2}\leq\epsilon\qquad\text{for all }\bQ'\text{ with }\|\bQ'-\bQ\|_{\rm op}\leq\delta.
\end{equation}
Therefore, for all such $\bQ'$,
\begin{align*}
    \|f_\star^{\rm eff}(\cdot;\bQ')-f_\star^{\rm eff}(\cdot;\bQ)\|_{L^2}
    &\leq \|f_\star^{\rm eff}(\cdot;\bQ')-q_\epsilon^{\rm eff}(\cdot;\bQ')\|_{L^2}+\|q_\epsilon^{\rm eff}(\cdot;\bQ')-q_\epsilon^{\rm eff}(\cdot;\bQ)\|_{L^2}\\
    &\quad+\|q_\epsilon^{\rm eff}(\cdot;\bQ)-f_\star^{\rm eff}(\cdot;\bQ)\|_{L^2} &&\text{(triangle inequality)}\\
    &\leq 3\epsilon &&\text{(by eqs.~\eqref{app:eq:contraction} and \eqref{app:eq:poly_continuity})}.
\end{align*}
As $\epsilon>0$ is arbitrary, the map $\bQ\mapsto f_\star^{\rm eff}(\cdot;\bQ)$ is continuous at $\bQ$.
\end{proof}

\begin{proof}[Proof of Lemma \ref{app:lemma:density}]
Since $w\mapsto f^\infty_w$ is linear, the set $\{f^\infty_w:\,w\in L^2_{\bz,b}\}$ is a linear subspace of $L^2(\cN(\bzero,\bI_r))$, and it is dense if and only if its orthogonal complement is $\{0\}$. Let $f\in L^2(\cN(\bzero,\bI_r))$ be such that $\E_{\bs}[f(\bs)f^\infty_w(\bs)]=0$ for all $w\in L^2_{\bz,b}$. We show that $f=0$. Define
\begin{equation}
    \psi_f(\bz,b):=\E_{\bs}\left[f(\bs)\,\sigma(\bz^\top\bs+b)\right],\qquad (\bz,b)\in\R^{r+1}.
\end{equation}
We use two properties of $\psi_f$. First, since $|\sigma|\leq B$, the Cauchy--Schwarz inequality gives $|\psi_f(\bz,b)|\leq B\|f\|_{L^2}$, so that $\psi_f\in L^2_{\bz,b}$. Second, $\psi_f$ is continuous on $\R^{r+1}$, by dominated convergence, since $\sigma$ is continuous and $|f(\bs)\sigma(\bz^\top\bs+b)|\leq B|f(\bs)|$, which is integrable. For any $w\in L^2_{\bz,b}$,
\begin{align*}
    0=\E_{\bs}\left[f(\bs)f^\infty_w(\bs)\right]&=\E_{\bs}\E_{\bz,b}\left[f(\bs)w(\bz,b)\sigma(\bz^\top\bs+b)\right]\\
    &=\E_{\bz,b}\left[w(\bz,b)\psi_f(\bz,b)\right],
\end{align*}
where the last step follows from Fubini's theorem, as $|f(\bs)w(\bz,b)\sigma(\bz^\top\bs+b)|\leq B|f(\bs)||w(\bz,b)|$. Choosing $w=\psi_f$ gives $\E_{\bz,b}[\psi_f(\bz,b)^2]=0$, and, since $\psi_f$ is continuous and the law of $(\bz,b)$ has a positive density on $\R^{r+1}$, $\psi_f(\bz,b)=0$ for all $(\bz,b)\in\R^{r+1}$. In other words, $f$ is orthogonal to all the functions $\bs\mapsto\sigma(\bz^\top\bs+b)$, hence to their linear span. Since $\sigma$ is bounded and nonconstant, Theorem 1 in \citet{hornik1991approximation} ensures that this span is dense in $L^2(\mu)$ for any finite measure $\mu$ on $\R^r$, in particular for $\mu=\cN(\bzero,\bI_r)$. Therefore, $f$ is orthogonal to the whole space $L^2(\cN(\bzero,\bI_r))$, i.e., $f=0$.
\end{proof}

\section{Proof of Theorem \ref{thm:genericNSE}}\label{app:sec:genericNSE}

In order to prove Theorem \ref{thm:genericNSE}, we consider the following matching upper and lower bounds for the optimal computational thresholds $\alpha_k^{\rm AMP}$ and the weighted mean-squared error in Definition~\ref{def:weighted_MSE}.
\subsection{Lower bound}In this section, we make use of the Bayes-AMP algorithm \citep{aubin2018committee,troiani2024fundamental}. Consider the modified dataset $\cD'_k = \{(\bx_i,y_i')\}_{i\in[n]}$, with single-index labels $y_i' = k^{-\gamma}g(\langle\bw_k^\star,\bx_i\rangle)+\sqrt{\Delta}\xi_i$ and the problem of learning $\bw_k^\star$ from $\cD'_k$. In \citet{celentano20a,montanari2022statistically} it is shown that the Bayes-AMP algorithm (Algorithm 1) for this single-index problem is provably optimal among first order algorithms, leading to the smallest mean-squared error. Consider now a second algorithm, which, knowing $\{\bw_h\}_{h\neq k}$, maps $\cD'_k$ to the original dataset $\cD$ with multi-index labels and applies the Bayes-AMP for multi-index models, in order to compute $\hat\bW^{\rm AMP}$ as an estimator for $\bW^\star$ (Algorithm 2). Note that this procedure is equivalent to applying Bayes-AMP to the original dataset. Concerning merely the estimation of $\bw_k^\star$, the Algorithm 2 is sub-optimal with respect to Algorithm 1, as the information on $\bw_k^\star$ is equivalent for both procedures. Therefore, the critical threshold and mean-squared error of Algorithm 1 are lower-bounds of the ones of Algorithm 2 in the optimal computational weak recovery of the $k^{\rm th}$ feature. Note that dataset $\cD'_k$ corresponds to the single-index problem studied in \citet{defilippis2026noise}, with SNR $\lambda = k^{-2\gamma}$. In particular, by Theorem 1 in \cite{defilippis2026noise}, for $k$ large enough, $ \alpha^{\rm AMP}_k $ is lower bounded by
\begin{equation}
   \alpha_k^{\rm SIM} = \Theta(k^{2\gamma\betast}).
\end{equation}
Define the lower bound estimator as $\bW^{\rm low}\in\R^{m_\star\times d}$ having each row $k$ obtained from Bayes-AMP applied to dataset $\cD'_k$.
For $\alpha$ large enough, the estimator $\hat\bW^{\rm low}$ has weakly recovered the first $\hat k(\alpha)= \Theta(\alpha^{1/(2\gamma\betast)})$ (if $\alpha\ll {m_\star}^{2\gamma\betast}$) and $\hat k(\alpha) = m_\star$ (if $\alpha\gg m_\star^{2\gamma\betast}$) features, with an asymptotic overlap $m^{\rm low}_k$. We can further lower bound the correspondent mean-squared error by considering, for each $k\in[\hat k(\alpha)]$, its statistical bottleneck, established in Corollary \ref{app:corollary:sindex_thresholds}. Consequently, the asymptotic error for the recovered features satisfies $1-(m^{\rm low}_k)^2 \geq C k^{2\gamma}/\alpha$ for $k\in[\hat k(\alpha)]$ and a certain constant $C$. For $\alpha\ll 
m_\star^{2\gamma\betast}$, the weighted mean-squared error is given by:
\begin{align*}
    {\rm MSE}_\gamma(\hat{\bW}^{\rm low}) &= \sum_{k=1}^{\hat k(\alpha)}k^{-2\gamma}(1-(m_k^{\rm low})^2) + \sum_{k=\hat k(\alpha)+1}^{m_\star} k^{-2\gamma}\\
    &\geq C\sum_{k=1}^{\hat k(\alpha)} k^{-2\gamma} \frac{k^{2\gamma}}{\alpha} + \Theta\left(\alpha^{(1-2\gamma)/(2\gamma\betast)}\right) \\
    &= C\frac{1}{\alpha}\sum_{k=1}^{\hat k(\alpha)} 1 + \Theta\left(\alpha^{(1-2\gamma)/(2\gamma\betast)}\right)\\
    &= \Theta\left(\alpha^{-1+1/(2\gamma\betast)}\right) + \Theta\left(\alpha^{-1/\betast (1-1/(2\gamma))}\right) .
\end{align*}
Instead, if $\alpha\gg m_\star^{2\gamma\betast}$,
\begin{align*}
    {\rm MSE}_\gamma(\hat{\bW}^{\rm low}) &= \sum_{k=1}^{m_\star}k^{-2\gamma}(1-(m_k^{\rm low})^2) \\
   &\geq  C\sum_{k=1}^{m_\star}k^{-2\gamma}\frac{k^{2\gamma}}{\alpha}\\
    &= \Theta(m_\star/\alpha).
\end{align*}

\subsection{Upper bound} 
An upper bound can be obtained as a direct adaptation of the proof in Appendix \ref{app:sec:spectral} to the case of generic noise sensitivity exponent $\betast\geq 1$. 

Note that the assumptions of Theorem~\ref{thm:genericNSE} do not fix the sign of the leading term in the first line of eq.~\eqref{app:eq:spectral_transition_system_approx}, nor of its analogue for $\beta_\star>1$, eq.~\eqref{app:eq:bound_spectral_scaling} below. The informative outliers can therefore emerge from either edge of the bulk (see Remark~\ref{app:remark:left_edge}). Here we consider the spectral method applied to both edges of the spectrum. Denoting by $\hat\bW^{\rm sp}_{+}\in\R^{r_+\times d}$ and $\hat\bW^{\rm sp}_{-}\in\R^{r_-\times d}$ the spectral estimators of Definition~\ref{def:spectral_estimator} computed from $\bT$ and from $-\bT$, respectively, we set $\hat\bW^{\rm sp}:=(\hat\bW^{\rm sp}_{+};\hat\bW^{\rm sp}_{-})\in\R^{(r_++r_-)\times d}$. Its rows are the eigenvectors of $\bT$ associated with its $r_+$ largest and $r_-$ smallest eigenvalues that are separated from the bulk, selected by the criterion of Definition~\ref{def:spectral_estimator}. Since $r_++r_-\leq\twidth$, we evaluate ${\rm MSE}_\gamma$ on the zero-padded estimator $(\hat\bW^{\rm sp},\bzero_{(\twidth-r_+-r_-)\times d})$, as in Theorem~\ref{thm:spectral_rates}.

Following the same steps as in Appendix \ref{app:sec:spectral}, one finds that, for any $h:\R\to\R$ bounded, there exists a constant $C$ such that, defining $G_k(y) = \E[z_k^2-1|Y=y]$, $Y\sim\cN(\sum_{k=1}^{m_\star} k^{-\gamma}g(z_k),\,\Delta)$, $z_k\sim\cN(0,1)$ for $k\in[m_\star]$,
\begin{equation}\label{app:eq:bound_spectral_scaling}
  \E[h(Y)G_k(Y)] = Ck^{-\gamma\betast}\E_z[g_k^{\betast}(z)(z^2-1)] + O_k(k^{-\gamma(\betast+1)}).
\end{equation}
Using this result, Theorem 4.1 and 4.2 in \citet{kovacevic25aspectral} imply that $\frac{1}{d}|\langle\hat{\bw}_k^{\rm sp},\bw_h^\star\rangle| = m^{\rm sp}_k\delta_{kh}$ and that $m^{\rm sp}_k>0$ if and only if $\alpha > \alpha_k^{\rm sp} $ where, for $k$ large enough, $\alpha_k^{\rm sp}$ is the solution of the following system of equations in $(\alpha_k^{\rm sp},t)$, 
\begin{equation}
    \begin{cases}
        (\alpha_k^{\rm sp})^{-1} = \Theta(t^{-1} k^{-\gamma\betast})\\
        (\alpha_k^{\rm sp})^{-1} = \Theta(t^{-2})
    \end{cases} \implies 
    \begin{cases}
        \alpha_k^{\rm sp} = \Theta(k^{2\gamma\betast})\\
        t = \Theta(k^{\gamma\betast}).
    \end{cases}
\end{equation}
Using Theorem 4.2 in \citet{kovacevic25aspectral}, we can also derive the leading order behavior of $m_k^{\rm sp}$. Consider the auxiliary function $\zeta_\alpha(t)$ eq. \eqref{app:eq:zeta_fun}
\begin{equation}
    \zeta_\alpha(t) = t\left(1+\alpha\E_Y\left[\frac{\cT(Y)}{t-\cT(Y)}\right]\right)
\end{equation}
Then, for $t\gg 1$
\begin{align}
    \zeta'_\alpha(t) &= 1 + \alpha\E_Y\left[\frac{\cT(Y)}{t-\cT(Y)}\right] -\alpha t \E_Y\left[\frac{\cT(Y)}{(t-\cT(Y))^2}\right]\\
    & = 1 - \alpha\E_Y\left[\frac{\cT^2(Y)}{(t-\cT(Y))^2}\right] \\
    &= 1 - \Theta(\alpha t^{-2}).
\end{align}
Further, defining $R_k(t) = \alpha t\,\E[z_k^2\cT(Y)/(t-\cT(Y))]$, for $t\gg 1$,
\begin{align*}
    R'_k(t) &= - \alpha\E\left[\frac{(z_k^2-1)\cT^2(Y)}{(t-\cT(Y))^2}\right] - \alpha\E\left[\frac{\cT^2(Y)}{(t-\cT(Y))^2}\right]\\
    &=-\Theta(\alpha k^{-\betast\gamma}t^{-2} + \alpha t^{-2}) &&\text{(by eq. \eqref{app:eq:bound_spectral_scaling})}\\
    &=-\Theta(\alpha t^{-2})
\end{align*}
Then, define $\hat t$ as the solution of $\alpha \E[G_k(Y)\cT(Y)/(t-\cT(Y))] = 1$, which, for $\alpha \gg 1$, $k$ large enough, using eq. \eqref{app:eq:bound_spectral_scaling}, is given by $\hat t = \Theta(\alpha k^{-\gamma\betast})$. By Theorem 4.2 in \citet{kovacevic25aspectral} applied to hierarchical models
\begin{align*}
    (m_k^{\rm sp})^2 &= \frac{\zeta'_\alpha(\hat t)}{\zeta'_\alpha(\hat t) - R'_k(\hat t)}\\
    &=\frac{1-\Theta(k^{2\gamma\betast}/\alpha)}{1-\Theta(k^{2\gamma\betast}/\alpha)+\Theta(k^{2\gamma\betast}/\alpha)}\\
    &=1-\Theta(k^{2\gamma\betast}/\alpha).
\end{align*}

As a consequence, denoting by $\hat k(\alpha)$ the number of weakly recovered features at sample complexity $\alpha\gg 1$, where $\hat k(\alpha) = \Theta(\alpha^{1/(2\gamma\betast)})$ if $\alpha\ll {\twidth}^{2\gamma\betast}$ and $\hat k (\alpha) = \twidth$ if $\alpha\gg {\twidth}^{2\gamma\betast}$, the weighted mean-squared error satisfies, for $\alpha \ll\ {m_\star}^{2\gamma\betast}$
\begin{align*}
     {\rm MSE}_\gamma(\hat{\bW}^{\rm sp}) &= \sum_{k=1}^{\hat k(\alpha)}k^{-2\gamma}(1-(m_k^{\rm sp})^2) + \sum_{k=\hat k(\alpha)+1}^{m_\star} k^{-2\gamma}\\
    &=\sum_{k=1}^{\hat k(\alpha)} k^{-2\gamma} \Theta\left(\frac{k^{2\gamma\betast}}{\alpha}\right) + \Theta\left(\alpha^{(1-2\gamma)/(2\gamma\betast)}\right) \\
    &=\Theta\left(\alpha^{(1-2\gamma)/(2\gamma\betast)}\right);
\end{align*}
while, for $\alpha\gg {m_\star}^{2\gamma\betast}$
\begin{align*}
     {\rm MSE}_\gamma(\hat{\bW}^{\rm sp}) &= \sum_{k=1}^{m_\star}k^{-2\gamma}(1-(m_k^{\rm sp})^2) \\
   &=  \sum_{k=1}^{m_\star}k^{-2\gamma}\Theta\left(\frac{k^{2\gamma\betast}}{\alpha}\right)\\
    &= \Theta({m_\star}^{1+2\gamma(\betast-1)}/\alpha).
\end{align*}

\section{Useful results on single-index models}\label{app:sec:results_sindex}
In this section we present various results applied in the derivation of the lower bound to ${\rm MMSE}_\gamma$ in Appendix \ref{app:sec:lower-bound}. In particular, we focus on the following single-index model setting.
\begin{definition}\label{app:def:single_index} Let $\bw^\star\sim\cN(\bzero_d,\bI_d)$, $g:\R\to\R$ satisfy Assumption \ref{assumption:gen_exp}. Consider the supervised learning problem of estimating $\bw_\star$ from $n$ i.i.d. observations $\cD = \{(\bx_i,y_i)\in\R^{d\times 1}:i\in[n]\}$ generated as
\begin{equation}
    \bx_i\sim\cN(\bzero,\bI_d/d),\qquad y_i = \sqrt{\lambda}g(\langle\bx_i,\bw^\star\rangle)+\xi_i,
\end{equation}
where $\lambda\geq 0$ is the signal-to-noise ratio (SNR) and $\xi_i\sim\cN(0,1)$ is additive noise.
\end{definition}
Since $\|\bw^\star\|^2/d\to1$ almost surely, in the high-dimensional limit this Gaussian prior is equivalent to the uniform prior on the sphere of radius $\sqrt d$ followed by the rows of $\bW_\star$, and the results below apply to both.
Given an estimator $\hat{\bw}$ of $\bw^\star$ that is a function of the dataset, we evaluate its estimation error using the following matrix-MMSE:
\begin{equation}
    {\rm mse}(\bw) := \frac{1}{d^2}\E[\|\bw\bw^\top-\bw^\star\bw^{\star\top}\|^2_{\rm F}],
\end{equation}
where $\E$ computes the expected value with respect to the joint distributions of $\cD$ and $\bw_\star$. A lower bound to this quantity is given by the following {\rm optimal} matrix-MSE
\begin{equation}\label{app:eq:mmse-sindex}
    {\rm mmse} := \frac{1}{d^2}\E\left[\|\E[\bw\bw^\top|\cD]-\bw^\star\bw^{\star\top}\|_{\rm F}^2\right] = \min_{\hat\bQ}\frac{1}{d^2}\E\left[\|\hat\bQ(\cD) - \bw^\star\bw^{\star\top}\|^2_{\rm F}\right].
\end{equation}
Intuitively, such optimal estimation error decreases with the SNR. 
\begin{lemma} \label{app:lemma:MMSE_SNR}
    In the setting of Definition~\ref{app:def:single_index}, the optimal matrix-MSE is a non-increasing function of the SNR $\lambda$.
\end{lemma}
\begin{proof}
Let $0\le\lambda'<\lambda$ and let $\cD=\{(\bx_i,y_i)\}_{i\in[n]}$ be generated at SNR $\lambda$. Define
\begin{equation*}
    y_i' = \sqrt{\lambda'/\lambda}\,y_i + \sqrt{1-\lambda'/\lambda}\,\tilde\xi_i,\qquad \tilde\xi_i\sim\cN(0,1)\ \text{i.i.d., independent of everything else.}
\end{equation*}
Then $y_i' = \sqrt{\lambda'}g(\langle\bx_i,\bw^\star\rangle)+\xi_i'$ with $\xi_i'\sim\cN(0,1)$ i.i.d. and independent of $(\bX,\bw^\star)$, so that $\cD'=\{(\bx_i,y_i')\}_{i\in[n]}$ has the same joint law with $\bw^\star$ as a dataset generated at SNR $\lambda'$. Since $\tilde{\boldsymbol\xi}$ is independent of $(\cD,\bw^\star)$, we have $\E[\bw\bw^\top|\cD,\tilde{\boldsymbol\xi}]=\E[\bw\bw^\top|\cD]$, and since $\cD'$ is a function of $(\cD,\tilde{\boldsymbol\xi})$, the optimality of the conditional expectation gives
\begin{equation*}
    {\rm mmse}(\lambda)=\frac{1}{d^2}\E\|\E[\bw\bw^\top|\cD,\tilde{\boldsymbol\xi}]-\bw^\star\bw^{\star\top}\|_{\rm F}^2\le\frac{1}{d^2}\E\|\E[\bw\bw^\top|\cD']-\bw^\star\bw^{\star\top}\|_{\rm F}^2={\rm mmse}(\lambda').
\end{equation*}
\end{proof}
We now consider the high-dimensional limit $n, d \to \infty$ with fixed ratio $\alpha = n/d$, referred to as the sample complexity. The following theorem specializes Theorems 1 and 2 of \citet{barbieroptimal2019} to the setting of interest.
\begin{theorem}[\citealp{barbieroptimal2019}] \label{app:thm:barbier} Consider the setting defined in \ref{app:def:single_index}. Then, in the limit $n,d\to\infty$, with fixed ratio $n/d=\alpha$,
\begin{equation}
    \frac{1}{d^2}\E\|\bw^\star\bw^{\star\top}-\E[\bw\bw\top\,|\,\cD]\|_{\rm F}^2\to 1 - m^2,
\end{equation}
and, given $\bw\sim\sP(\cdot|\cD)$,
\begin{equation}
    \frac{1}{d}|\bw^\top\bw^\star|\xrightarrow{\mathbb P}m,
\end{equation}
where $m = m(\alpha)$ is the maximizer of 
\begin{align}
\label{app:eq:free_entropy}&\sup_{m\in[0,1]}f_{\rm RS}(m),\quad f_{\rm RS}(m) := \left\{m + \log(1-m) + 2\alpha \Psi_{\rm out}(m)\right\},\\& \Psi_{\rm out}(m;\lambda) := \E_{W,V,Y}\log \E_{w\sim\cN(0,1)}\left[\mathsf P(Y|\sqrt{m}V+\sqrt{1-m} w)\right], \label{app:eq:psi_out}
\end{align}
with $V,W\sim\cN(0,1)$, $Y\sim\mathsf P(\cdot|\sqrt{m}V+\sqrt{1-m}W)$. 
\end{theorem}

As a direct consequence of the theorem, the information-theoretic weak recovery threshold $\alpha^{\rm IT}$ is the smallest sample complexity $\alpha$ such that the maximizer of the free entropy eq. \eqref{app:eq:free_entropy} is $m \neq 0$. Equivalently, $\alpha^{\rm IT}$ is the smallest $\alpha$ such that ${\rm mmse}<1$. 
Lemma \ref{app:lemma:MMSE_SNR} readily implies that $\alpha^{\rm IT} = \alpha^{\rm IT}(\lambda)$ is non-increasing in the SNR.\\
Finally, we characterize the IT weak recovery threshold in the limit of small SNR. By expanding eq. \eqref{app:eq:psi_out} around $\lambda = 0$, we derive the following Corollary.

\begin{corollary}[IT weak recovery threshold in the large noise regime]\label{app:corollary:sindex_thresholds} Consider the setting of Theorem \ref{app:thm:barbier}. Then, in the limit $\lambda\to 0$, the information theoretic weak-recovery threshold satisfies
\begin{equation}
    \alpha^{\rm IT} = \Theta(\lambda^{-1})
\end{equation}
and, for $\alpha\to\infty$, and fixed $\lambda$ the optimal matrix-MSE scales as
\begin{equation}
    {\rm mmse} = \Theta\left(\frac{1}{\lambda\alpha}\right).
\end{equation}
\end{corollary}
\begin{proof} For a fixed $m \in[0,1])$, consider the function 
\begin{align}
F(\kappa) := \E_{W,V,Y}\log \E_{w\sim\cN(0,1)}\left[\exp\left(-(y - \kappa g(\sqrt{m}V+\sqrt{1-m}w))^2/2\right)\right],
\end{align}
with $V,W\sim\cN(0,1)$, $Y\sim\cN(\kappa g(\sqrt{m}V+\sqrt{1-m}W), 1)$. It is straightforward to show, by a change of variable $Y\to-Y$, that $F$ is even. Note that, up to constant terms in $m$, $F(\kappa) = \Psi_{\rm out}(m;\kappa^2)$.

Furthermore, as a consequence of Taylor's theorem, there exists constant $\hat\kappa\in[0,\kappa]$, such that $F(\kappa) = F(0) + \frac{1}{2}F''(0)\kappa^2 + \frac{1}{24}F^{(4)}(\hat \kappa)\kappa(\kappa-\hat\kappa)^3$. By Lemma \ref{app:lemma:derivatives}, there exists a finite $C>0$, constant in $\kappa$ and $m$, such that 
\begin{equation}
    \left|\frac{1}{24}F^{(4)}(\hat \kappa)\kappa(\kappa-\hat\kappa)^3\right| < C \kappa^4.
\end{equation}
Define the auxiliary function
\begin{equation}
    Z(\kappa, v, u, y) := \frac{1}{\sqrt{2\pi}}\E_w\left[\exp\left(-(y-\kappa (g(\sqrt{m}v+\sqrt{1-m}w)- g(\sqrt{m}v+\sqrt{1-m}u)))^2/2\right)\right].
\end{equation}
In particular, $Z(0,v,u,y) = (2\pi)^{-1/2}e^{-y^2/2}$ and, applying the definition of the probabilist's Hermite polynomials, for all $j\in\mathbb N$, using the shorthand $g_u = g(\sqrt{m}v+\sqrt{1-m}u)$
\begin{align}
    \frac{\partial^j}{\partial \kappa^j}Z(\kappa,v,u,y) = \frac{1}{\sqrt{2\pi}}\E_w\left[e^{-(y-\kappa (g_w-g_u))^2/2} \He_j(y-\kappa (g_w-g_u))\frac{(g_w-g_u)^j}{j!}\right]
\end{align}
we have, after a change of variable $Y\to Y + \kappa g\sqrt{m}V+\sqrt{1-m}W)$, that  $F(\kappa) = \E_{V,W,Y}[\log Z(\kappa,V,W,Y)]$ for $Y\sim\cN(0,1)$ and
\begin{align}
F''(0) &= \E_{V,W,Y}\left[\frac{Z''(0,V,W,Y)}{Z(0,V,W,Y)} - \left(\frac{Z'(0,V,W,Y)}{Z(0,V,W,Y)}\right)^2\right]\\
& =\E_Y[Y](\E_{V,W}[g_w]-\E_{V,w}[g_w]) \nonumber\\
& \quad - \frac{1}{2}\E_Y[(Y^2-1)]\left(\E_{V,W}[g_w^2] + \E_{V,W}[g_w^2] - 2\E_{V,W,w}[g_Wg_w]) \right)\\
&= 0 - \E_{z\sim\cN(0,1)}[g^2(z)] + \E_{(z_1,z_2)\sim\cN(\bzero_2,\bC)}[g(z_1)g(z_2)]
\end{align}
with
\begin{align}
    \bC = \left(\begin{array}{cc}
        1 & m \\
        m & 1
    \end{array}\right).
\end{align}
In the above we used 
\begin{align}
    \E_{V,W}[g^2_W] &= \E_{W,V}[g^2(\sqrt{m}V+\sqrt{1-m}W)] = \E_{z}[g^2(z)],\\
    \E_{V,w,W}[g_wg_W] &= \E_V[\E_{W}[g(\sqrt{m}V+\sqrt{1-m}W)]\E_{w}[g(\sqrt{m}V+\sqrt{1-m}w)]]\\
    &=\E_{(z_1,z_2)\sim\cN(\bzero_2,\bC)}[g(z_1)g(z_2)].
\end{align}
Therefore, up to constant terms in $m$,\footnote{Recall that, by Lemma \ref{app:lemma:derivatives}, the correction $O(\lambda^2)$ is uniform in $m\in[0,1]$.}

\begin{equation}
    \Psi_{\rm out}(m; \lambda) = F(\sqrt{\lambda}) = {\rm const} + \frac{\lambda}{2}\E_{(z_1,z_2)\sim\cN(\bzero_2,\bC)}[g(z_1)g(z_2)] +O(\lambda^2).
\end{equation}
Assumption \ref{assumption:gen_exp} ensures that $g$ can be decomposed in the Hermite basis as 
\begin{equation}
    g(z) = \sum_{k\geq 0}\frac{c_k}{k!}\He_k(z),\qquad c_k := \E_{z\sim\cN(0,1)}\left[g(z)\He_k(z)\right].
\end{equation}
Leveraging Proposition 11.31 in \citet{ODonnell_2014},
\begin{equation}\label{eq:hermite_correlation}
    \E_{(z_1,z_2)\sim\cN(\bzero_2,\bC)}[g(z_1)g(z_2)] = \sum_{k\geq 0} \frac{c_k^2}{k!} m^k\in [0,\E_z[g^2(z)]],
\end{equation}
which is a non-decreasing function of $m\in[0,1]$. Note that, since $g$ has generative exponent 2, we have that $c_1=0$ necessarily. Moreover, by Assumption \ref{assumption:gen_exp}, $c_2\neq 0$ and bounded. Hence, we are interested in maximizing the quantity\footnote{Note that we are neglecting constant terms with respect to $m$.}
\begin{align}\label{app:eq:free_entropy_m}
   f_{\rm RS}(m) &:= m + \log (1-m) +{\alpha \lambda }\sum_{k\geq 2} \frac{c_k^2}{k!} m^k + O(\alpha\lambda^2)\\
   &=\frac{1}{2}\left({\alpha\lambda c_2^2} -1 \right)m^2+\sum_{k >2}\left(\frac{\alpha\lambda}{k!}c_k^2 - \frac{1}{k}\right)m^k + O(\alpha\lambda^2),
\end{align}
where we have expanded $\log(1-m)$.
For $\alpha > \lambda^{-1} c_2^{-2}$, $m=0$ is a minimum of the free entropy, therefore $\alpha^{\rm IT} \leq \lambda^{-1} c_2^{-2}$. 
Denote 
\begin{equation}
 D := \inf_{m\in(0,1]}\frac{-m-\log(1-m) }{\sum_{k}\frac{c_k^2}{k!}  m^k}, 
\end{equation}
which is strictly positive and well-defined. Note that 
\begin{equation}
    \lim_{m\to 0^+}\frac{-m-\log(1-m) }{\sum_{k}\frac{c_k^2}{k!}  m^k} = c_2^{-2}\implies \frac{1}{\E_z[g^2(z)]}\leq D \leq c_2^{-2} = \frac{1}{\E_z[g''(z)]^2}.
\end{equation}
Then, for all $\alpha < D\lambda^{-1}$, $f(m\neq 0)<f(0) = 0$, {\it i.e.} $m=0$ is the global maximizer and
\begin{equation}
    D \leq \alpha^{\rm IT} \lambda \leq c_2^{-2} \implies \alpha^{\rm IT} = \Theta(\lambda^{-1}),
\end{equation}
For the second result, we first consider the setting $\lambda\to0$ and $\alpha\to \infty$, with $\alpha\gg\lambda$. There exists a non-zero maximizer $m$, which satisfies\footnote{In the following we retain the leading terms in the free entropy expansion.} 
\begin{equation}
    \frac{\de}{\de m}f_{\rm RS}(m) \overset{!}{=}0 \implies \frac{1}{1-m} = \alpha\lambda\sum_{k\geq 2}\frac{c_k^2}{k!}m^{k-2}\overset{}{\implies} \sum_{k\geq 2}\frac{c_k^2}{k!}\left(m^{k-2}-m^{k-1}\right) = \frac{1}{\alpha\lambda},
\end{equation}
The equation is solved by $m = 1 - \Theta\left(\frac{1}{\alpha\lambda}\right)$. Theorem \ref{app:thm:barbier} implies that, in this regime, ${\rm mmse}(\lambda) = \Theta((\alpha\lambda)^{-1})$. Together Lemma \ref{app:lemma:MMSE_SNR}, the result for arbitrary $\lambda$ is proved.
\end{proof}
\begin{remark}\label{app:remark:c2_zero}
The scaling $\alpha^{\rm IT}=\Theta(\lambda^{-1})$ in Corollary \ref{app:corollary:sindex_thresholds} also holds when $c_2=0$, which is the case of a noise sensitivity exponent larger than one \citep{defilippis2026noise}. The lower bound $\alpha^{\rm IT}\geq D\lambda^{-1}$ is unchanged, as $D>0$ also in this case. For the upper bound, fix $\hat m\in(0,1)$ and let $C:=-\log(1-\hat m)/\sum_{k}\frac{c_k^2}{k!}\hat m^k$. Then, at $\alpha=C\lambda^{-1}$, eq.~\eqref{app:eq:free_entropy_m} gives $f_{\rm RS}(\hat m)= \hat m+O(\lambda)>0=f_{\rm RS}(0)$ for $\lambda$ small enough, so that $m=0$ is not the global maximizer, and $\alpha^{\rm IT}\leq C\lambda^{-1}$.
\end{remark}
In the proofs of Corollary \ref{app:corollary:sindex_thresholds} and Remark \ref{app:remark:c2_zero} we have used the following lemma.
\begin{lemma}\label{app:lemma:derivatives} Let $m \in[0,1]$ and consider the function $Z:\R^4\to\R$
\begin{align}
Z(\kappa,v,w,y) &:= \frac{1}{\sqrt{2\pi}}\E_{U\sim\cN(0,1)}\left[e^{-\left(y - \kappa (g(\sqrt{m}v+\sqrt{1-m}U-g(\sqrt{m}v+\sqrt{1-m}w))\right)^2/2}\right].
\end{align}
Then, for all $k\in\mathbb N$, there exists a universal constant $C_k$ such that
\begin{equation}
    \E_{W,V,Y}\left|\frac{\partial^k}{\partial\kappa^k}\log Z(\kappa,V,W,Y)\right| < C_k.
\end{equation}
with $V,W,Y\sim\cN(0,1)$.     
\end{lemma}
\begin{proof}
In order to simplify the notation, given $V\sim\cN(0,1)$, define the shorthand $g_w = g(\sqrt{m}V+\sqrt{1-m}w)$. 
By Faà di Bruno's formula, using the shorthand $Z = Z(\kappa, v, w,y )$ we have that 
\begin{align}
    \frac{\partial^k}{\partial \kappa^k}\log Z &= \sum_{\boldm:\sum_j j m_j=k}\frac{k!}{m_1!\ldots m_k!}\frac{(-1)^{|\boldm|}}{|\boldm| Z^{|\boldm|}}\prod_{j=1}^k \left(\frac{\partial^{(j)}_\kappa Z}{j!}\right)^{m_j},\\
    &= \sum_{\boldm:\sum_j j m_j=k}\frac{k!}{m_1!\ldots m_k!}\frac{(-1)^{|\boldm|}}{|\boldm| }\prod_{j=1}^k \left(\frac{\partial^{(j)}_\kappa Z}{j!Z} \right)^{m_j},
\end{align}
with $\boldm\in\mathbb N^k$ and $|\boldm|=\sum_j m_j$. Therefore, applying the triangle inequality and Holder's inequality
\begin{align}
    \E \,\left|\frac{\partial^k}{\partial \kappa^k}\log Z\right| \leq \sum_{\boldm:\sum_j j m_j=k}\frac{k!}{m_1!\ldots m_k!}\frac{(-1)^{|\boldm|}}{|\boldm| }\prod_{j=1}^k \left(\E\,\left|\frac{\partial^{(j)}_\kappa Z}{j!Z} \right|^{k m_j}\right)^{1/k}.
\end{align}
Since the sum runs over a finite number of $k$-tuples, the problem can be reduced to showing the finiteness of the quantities
\begin{equation}
    \E_{V,W,Y}\left|\frac{\frac{\partial^j}{\partial \kappa^j}Z(\kappa,V,W,Y)}{Z(\kappa,V,W,Y)}\right|^p, \qquad p\in\mathbb N
\end{equation}
Recalling the definition of the probabilist's Hermite polynomials:
\begin{align}
\frac{1}{Z}\frac{\partial^j}{\partial\kappa^j}Z &= \frac{1}{Z\sqrt{2\pi}}\E_U\left[e^{-(y+\kappa (g_W-g_U))^2/2}\He_j(y+\kappa (g_W-g_U))\frac{(g_W-g_U)^j}{j!}\right]\\
&=\E_{U\sim \sfQ}\left[\He_j(y+\kappa (g_W-g_U))\frac{(g_W-g_U)^j}{j!}\right],
\end{align}
where we have defined the probability density function
\begin{equation}
    \de \sfQ(U)= \frac{\exp\left(-U^2/2-(y+\kappa (g_U-g_W))^2/2\right)}{\E_{U\sim\cN(0,1)}\left[\exp\left(-U^2/2-(y+\kappa (g_W-g_U))^2/2\right)\right]}\de U.
\end{equation}
Applying Jensen's inequality to the convex function $|\cdot|^p$ we obtain
\begin{align}
    &\E_{V,W,Y}\left|\frac{\partial_\kappa^{(j)}Z(\kappa,V,W,Y)}{Z(\kappa,V,W,Y)}\right|^p \\\leq &\E_{V,W,Y}\E_{U\sim \sfQ}\left|\He_j(y+\kappa (g_W-g_U))\frac{(g_W-g_U)^j}{j!}\right|^p
\end{align}
using the change of variable $y\to y-\kappa g_W$
\begin{align}
{=} &\E_V\int\de y \E_W\frac{e^{-(y-\kappa g_W)^2/2}}{\sqrt{2\pi}}\frac{\E_{U\sim\cN(0,1)}\left[e^{-(y-\kappa g_U)^2/2}\left|\He_j(y-\kappa g_U)\frac{(g_W-g_U)^j}{j!}\right|^p\right]}{\E_U\left[{e^{-(y-\kappa g_U)^2/2}}\right]}\\
\leq& 2^{jp-1}\E_V\int\de y \E_W\frac{e^{-(y-\kappa g_W)^2/2}}{\sqrt{2\pi}}\frac{\E_{U\sim\cN(0,1)}\left[e^{-(y-\kappa g_U)^2/2}\left|\He_j(y-\kappa g_U)\right|^p\frac{\left|g_W\right|^{jp}+\left|g_U\right|^{jp}}{j!^p}\right]}{\E_U\left[{e^{-(y-\kappa g_U)^2/2}}\right]}\\
=&\frac{2^{jp-1}}{j!^p}\underbrace{\E_V\int\de y \E_{W\sim \sfQ}\left[\left|g_W\right|^{jp}\right]\frac{\E_{U\sim\cN(0,1)}\left[e^{-(y-\kappa g_U)^2/2}\left|\He_j(y-\kappa g_U)\right|^p\right]}{\sqrt{2\pi}}}_{=:T_1}\\
&+\frac{2^{jp-1}}{j!}\underbrace{\E_V\int\de y \E_U\frac{e^{-(y-\kappa g_U)^2/2}}{\sqrt{2\pi}}\frac{\E_{U\sim\cN(0,1)}\left[e^{-(y-\kappa g_U)^2/2}\left|\He_j(y-\kappa g_w)\right|^p\,{\left|g_w\right|^{jp}}\right]}{\E_U\left[{e^{-(y-\kappa g_U)^2/2}}\right]}}_{=:T_2}
\end{align}
In the first term $T_1$, applying the Cauchy-Shwartz,
\begin{align}
    T_1 &\propto \E_{V,U}\int \de y  \E_{W\sim \sfQ} \left(e^{-(y-\kappa g_U)^2/4} \left|g_W\right|^{jp} \right)\left(e^{-(y-\kappa g_U)^2/4}\left|\He_j(y-\kappa g_U)\right|^p\right)\\
    &\leq \left(\E_V\int\de y\E_U [e^{-(y-\kappa g_U)^2/2} ]\E_{W\sim \sfQ}\left[\left|g_W\right|^{2jp}\right] \right)^{1/2}\nonumber\\
    &\quad\times\left(\E_V\int\de y\E_U e^{-(y-\kappa g_U)^2/2} \left|\He_j(y-\kappa g_U)\right|^{2p}\right)^{1/2}  \\
    & = \left(\E_V\int\de y \E_{W\sim\cN(0,1)}\left[e^{-(y-\kappa g_W)^2/2}\left|g_W\right|^{2jp}\right] \right)^{1/2}\nonumber\\
    &\quad\times\left(\int\de ye^{-y^2/2} \left|\He_j(y)\right|^{2p}\right)^{1/2} \\
     & = \left(\E_V\int\de y e^{-y^2/2}\E_{W\sim\cN(0,1)}\left[\left|g_W\right|^{2jp}\right] \right)^{1/2}\left(\int\de ye^{-y^2/2} \left|\He_j(y)\right|^{2p}\right)^{1/2}\\
    & = \left(\sqrt{2\pi}\E_{z\sim\cN(0,1)}[|g(z)|^{2jp}]\right)^{1/2}\left(\int\de ye^{-y^2/2} \left|\He_j(y)\right|^{2p}\right)^{1/2} < \infty,
\end{align}
where the first factor is bounded as $|g(x)| < C(1 + |x|^b)$, for some constants $C\in\R$, $b\in\mathbb N$ (Definition \ref{app:def:single_index}). Moreover, we have used
\begin{equation}
    \E_{V,W}[g^2_W] = \E_{W,V}[g^2(\sqrt{m}V+\sqrt{1-m}W)] = \E_{z}[g^2(z)].
\end{equation}
Similarly,
\begin{align}
    T_2 &\propto \E_V\int\de y\E_{U}\left[e^{-(y-\kappa g_U)^2/2}\left|\He_j(y-\kappa g_U)\right|^p\,{\left|g_U\right|^{jp}}\right]\\
    &= \left(\E_{z\sim\cN(0,1)}[|g(z)|^{jp}]\right)\left(\int\de ye^{-y^2/2} \left|\He_j(y)\right|^{p}\right)<\infty
\end{align}
Both $T_1$ and $T_2$ are bounded by quantities independent of $\kappa$ and $m$, and that are functions of the finite moments $\E_z[g^k(z)]$, which completes the proof.
\end{proof}
\end{document}